\documentclass[10pt]{article} 

\relax

\usepackage{fullpage} 

\usepackage{times}  
\usepackage{helvet} 
\usepackage{courier}  
\usepackage[hyphens]{url}  
\usepackage{graphicx} 
\usepackage{natbib}  
\usepackage{caption} 
\title{Improved Worst-Case Regret Bounds for Randomized Least-Squares Value Iteration}
\author{
Priyank Agrawal\footnote{These two authors contributed equally.}, Jinglin Chen\footnotemark[1], Nan Jiang\\ 
University of Illinois at Urbana-Champaign
}

 \usepackage{hyperref} 
\usepackage[utf8]{inputenc} 
\usepackage{lmodern}
\usepackage{url}            
\usepackage{booktabs}       %
\usepackage{amsfonts}       
\usepackage{nicefrac}       
\usepackage{xcolor}
\usepackage{commath}
\usepackage{graphicx}
\usepackage{comment}
\usepackage{microtype}      
\usepackage{amsthm, amsmath, amssymb}
\usepackage{thm-restate}
\usepackage{algorithm, algorithmic}
\usepackage{bbm}
\usepackage[switch]{lineno}

\newcommand{\E}{\mathbb{E}}
\newcommand{\Prob}{\mathbb{P}}

\newcommand{\Ac}{\mathcal{A}}
\newcommand{\Sc}{\mathcal{S}}

\usepackage{longtable}

\newcommand{\vc}[3]{\overset{#2}{\underset{#3}{#1}}}  

\DeclareMathOperator*{\argmin}{argmin}

\def\ind{\mathbf{1}}

\usepackage{mathtools}
\newcommand{\barV}{\overline{V}}
\newcommand{\tilV}{\tilde{V}}

\newcommand{\underV}{\underline{V}}

\newcommand{\barM}{\overline{M}}
\newcommand{\barQ}{\overline{Q}}
\newcommand{\underQ}{\underline{Q}}

\newcommand{\Rh}{R_{h,s_h,a_h}}
\newcommand{\RHath}{\hat{R}_{h,s_h,a_h}}
\newcommand{\Ph}{P_{h,s_h,a_h}}
\newcommand{\PHath}{\hat{P}_{h,s_h,a_h}}
\newcommand{\Rhk}{R_{h,s^k_h,a^k_h}}
\newcommand{\RHathk}{\hat{R}^k_{h,s^k_h,a^k_h}}
\newcommand{\Phk}{P_{h,s^k_h,a^k_h}}
\newcommand{\PHathk}{\hat{P}^k_{h,s^k_h,a^k_h}}
\newcommand{\noiseh}{w_{h,s_h,a_h}}
\newcommand{\noisehU}{\underline{w}_{h,s_h,a_h}}
\newcommand{\nh}{n(h,s_h,a_h)}

\newcommand{\R}{R_{h,s,a}}

\newcommand{\PHat}{\hat{P}_{h,s,a}}
\newcommand{\Rk}{R_{h,s,a}}
\newcommand{\RHatk}{\hat{R}^k_{h,s,a}}
\newcommand{\Pk}{P_{h,s,a}}
\newcommand{\PHatk}{\hat{P}^k_{h,s,a}}

\newcommand{\noisehk}{w^k_{h,s^k_h,a^k_h}}
\newcommand{\noisehUk}{\underline{w}^k_{h,s^k_h,a^k_h}}
\newcommand{\nhk}{n^k(h,s^k_h,a^k_h)}

\newcommand{\Rhkmain}{R^k_{h}}
\newcommand{\RHathkmain}{\hat{R}^k_{h}}
\newcommand{\Phkmain}{P^k_{h}}
\newcommand{\PHathkmain}{\hat{P}^k_{h}}
\newcommand{\noisehkmain}{w^k_{h}}
\newcommand{\noisehUkmain}{\underline{w}^k_{h}}
\newcommand{\nhkmain}{n^k(h)}

\newcommand{\PhPi}{P_{h,s_h,\pi_h(s_h)}}
\newcommand{\PHathPi}{\hat{P}_{h,s_h,\pi_h(s_h)}}

\newcommand{\PhPik}{P^k_{h,s^k_h,\pi^k_h(s^k_h)}}

\newcommand{\deltaEO}[1]{\overline{\delta}_{#1}(s_{#1})}
\newcommand{\deltaEPi}[1]{\overline{\delta}^{\pi}_{#1}(s_{#1})}
\newcommand{\deltaPiU}[1]{\underline{\delta}^{\pi}_{#1}(s_{#1})}

\newcommand{\deltaEPik}[1]{\overline{\delta}^{\pi^k}_{#1,k}(s^k_{#1})}
\newcommand{\deltaPiUk}[1]{\underline{\delta}^{\pi^k}_{#1,k}(s^k_{#1})}

\newcommand{\decompC}{\sqrt{\frac{1}{4\del{C_2+1}^2H^2L}}}
\newcommand{\nextStateC}{4H^2L\del{
C_2+1}^2}
\newcommand{\decompCtwo}{4\del{C_2+1}^2SH^2L}
\newcommand{\MDSa}{\mathcal{M}_{\envert{\deltaEO{h}}}}
\newcommand{\MDSb}{\mathcal{M}_{\deltaEPi{h}}}
\newcommand{\MDSc}{\mathcal{M}_{\deltaPiU{h}}}

\newcommand{\MDSbk}{\mathcal{M}_{\deltaEPik{h}}}
\newcommand{\MDSck}{\mathcal{M}_{\deltaPiUk{h}}}

\newcommand{\HistoryN}{\overline{\mathcal{H}}^k_{h} }
\newcommand{\History}{\mathcal{H}^k_{h} }

\newcommand{\noiseind}[1]{w_{#1,s_{#1},a_{#1}}}
\newcommand{\noiseUind}[1]{\underline{w}_{#1,s_{#1},a_{#1}}}

\newcommand{\RDiffh}{\mathcal{R}_{h,s_h,a_h}}
\newcommand{\PDiffh}{\mathcal{P}_{h,s_h,a_h}}

\newcommand{\RDiffind}[1]{\mathcal{R}_{#1,s_{#1},a_{#1}}}
\newcommand{\PDiffind}[1]{\mathcal{P}_{#1,s_{#1},a_{#1}}}

\newcommand{\MDSbind}[1]{\mathcal{M}_{\deltaEPi{#1}}}
\newcommand{\MDScind}[1]{\mathcal{M}_{\deltaPiU{#1}}}

\newcommand{\MDSfind}[1]{\mathcal{M}^w_{#1}}

\newcommand{\MDSfkind}[1]{\mathcal{M}^w_{#1,k}}

\newcommand{\RDiffhk}{\mathcal{R}^k_{h,s^k_h,a^k_h}}
\newcommand{\PDiffhk}{\mathcal{P}^k_{h,s^k_h,a^k_h}}
\newcommand{\OptPr}{\Phi(-\sqrt 2)/2}
\newcommand{\Cthreshold}{2\sqrt{H^3S\log\del{2HSAk}\decompCL} }
\newcommand{\CthresholdSquared}{4H^3S\log\del{2HSAk}\decompCL }
\newcommand{\decompCL}{\log\del{40SAT/\delta}}
\newcommand{\noiseSigma}{\frac{H^3S\log(2HSAk)}{2(n^k(h,s,a)+1)}}
\newcommand{\noiseBeta}{H^3S\log(2HSAk)}

\newcommand{\algo}{\textsc{C-RLSVI}}

\newcommand{\Sn}{S_{\text{next}}}
\newcommand{\Econf}{\mathcal{C}_k}
\newcommand{\Ethreshk}{\mathcal{E}^{\text{th}}_{h,k}}
\newcommand{\Ethresk}{\mathcal{E}^{\text{th}}_{k}}
\newcommand{\EthreskC}{\mathcal{E}^{\text{th}\,\complement}_{k}}

\newcommand{\Eopt}{\overline{\mathcal{O}}_{1,k}}
\newcommand{\EoptTwo}{\tilde{\mathcal{O}}_{1}}
\newcommand{\EoptTwohk}{\tilde{\mathcal{O}}_{1,k}}
\newcommand{\polbarM}{\pi}
\newcommand{\polbarMk}{}

\newcommand{\Scal}{\mathcal{S}}
\newcommand{\Acal}{\mathcal{A}}
\newcommand{\EE}{\mathbb{E}}
\newcommand{\RR}{\mathbb{R}}
\newcommand{\Gcal}{\mathcal{G}}

\newcommand{\EstErrorhk}{\epsilon_{h,k}^{\text{err}}}

\newtheorem{thm}{Theorem}
\newtheorem{lemma}[thm]{Lemma}

\newtheorem{definition}{Definition}

\allowdisplaybreaks

\begin{document}

\clearpage
\maketitle

\begin{abstract}
This paper studies regret minimization with randomized value functions in reinforcement learning. In tabular finite-horizon Markov Decision Processes, we introduce a clipping variant of one classical Thompson Sampling (TS)-like algorithm, randomized least-squares value iteration (RLSVI). Our $\tilde{\mathrm{O}}(H^2S\sqrt{AT})$ high-probability worst-case regret bound improves the previous sharpest worst-case regret bounds for RLSVI and matches the existing state-of-the-art worst-case TS-based regret bounds.
\end{abstract}

\section{Introduction}\label{sec: introduction main}
We study systematic exploration in reinforcement learning (RL) and the exploration-exploitation trade-off therein. Exploration in RL \cite{sutton2018reinforcement} has predominantly focused on \textit{Optimism in the face of Uncertainty} (OFU) based algorithms. Since the seminal work of~\citet{jaksch2010near}, many provably efficient methods have been proposed but most of them are restricted to either tabular or linear setting~\citep{azar2017minimax,jin2020provably}. A few paper study a more general framework but subjected to computational intractability~\citep{jiang2017contextual,sun2019model,henaff2019explicit}. Another broad category is Thompson Sampling (TS)-based methods \citep{osband2013more,agrawal2017optimistic}. They are believed to have more appealing empirical results \citep{chapelle2011empirical,osband2017posterior}.\\ 

In this work, we investigate a TS-like algorithm, RLSVI~\citep{osband2016generalization,osband2019deep,russo2019worst,zanette2020frequentist}. In RLSVI, the exploration is induced by injecting randomness into the value function. The algorithm generates a randomized value function by carefully selecting the variance of Gaussian noise, which is used in perturbations of the history data (the trajectory of the algorithm till the current episode) and then applies the least square policy iteration algorithm of~\citet{lagoudakis2003least}. Thanks to the model-free nature, RLSVI is flexible enough to be extended to general function approximation setting, as shown by~\citet{osband2016deep,osband2018randomized,osband2019deep}, and at the same time has less burden on the computational side.\\

We propose \algo\, algorithm, which additionally considers an initial burn-in or warm-up phase on top of the core structure of RLSVI. Theoretically, we prove that \algo\, achieves $\tilde{\mathrm{O}}(H^2S\sqrt{AT})$ high-probability regret bound\footnote{$\tilde{\mathrm{O}}\del{\cdot}$ hides dependence on logarithmic factors.}

\paragraph{Significance of Our Results}
\begin{itemize}
    \item Our high-probability bound improves upon previous $\tilde{\mathrm{O}}(H^{\nicefrac{5}{2}}S^{\nicefrac{3}{2}}\sqrt{AT})$ worst-case expected regret bound of RLSVI in \citet{russo2019worst}.
    \item Our high-probability regret bound matches the sharpest $\tilde{\mathrm{O}}(H^2S\sqrt{AT})$ worst-case regret bound among all TS-based methods~\citep{agrawal2017optimistic}\footnote{\citet{agrawal2017optimistic} studies weakly communicating MDPs with diameter $D$. Bounds comparable to our setting (time in-homogeneous) are obtained by augmenting their state space as $S'\rightarrow SH$ and noticing $D \ge H$.}.
\end{itemize}

\paragraph{Related Works} Taking inspirations from~\citet{azar2017minimax,dann2017unifying,zanette2019tighter,yang2019reinforcement}, we introduce clipping to avoid propagation of unreasonable estimates of the value function. Clipping creates a warm-up effect that only affects the regret bound with constant factors (i.e. independent of the total number of steps $T$). With the help of clipping, we prove that the randomized value functions are bounded with high probability.\\

In the context of using perturbation or random noise methods to obtain provable exploration guarantees, there have been recent works~\citep{osband2016deep,fortunato2018noisy,pacchiano2020optimism,xu1001worst,kveton2019garbage} in both theoretical RL and bandit literature. A common theme has been to develop a TS-like algorithm that is suitable for complex models where exact posterior sampling is impossible. RLSVI also enjoys such conceptual connections with Thompson sampling~\citep{osband2019deep,osband2016generalization}. Related to this theme, the worst-case analysis of~\citet{agrawal2017optimistic} should be highlighted, where the authors do not solve for a pure TS algorithm but have proposed an algorithm that samples many times from posterior distribution to obtain an optimistic model. In comparison, \algo\, does not require such strong optimistic guarantee.\\

Our results are not optimal as compared with $\mathrm{\Omega}({H\sqrt{SAT})}$ lower bounds in~\citet{jaksch2010near} \footnote{The lower bound is translated to time-inhomogeneous setting.}. The gap of $\sqrt{SH}$ is sometimes attributed to the additional cost of exploration in TS-like approaches~\citep{abeille2017linear}. Whether this gap can be closed, at least for RLSVI, is still an interesting open question. We hope our analysis serves as a building block towards a deeper understanding of TS-based methods.

\section{Preliminaries}\label{sec: preliminaries main}
\paragraph{Markov Decision Processes} 
We consider the episodic Markov Decision Process (MDP) $M=(H,\mathcal{S},\mathcal{A},P,R,s_1)$ described by \citet{puterman2014markov}, where $H$ is the length of the episode, $\mathcal{S}=\{1,2,\ldots, S\}$ is the finite state space, $\mathcal{A}=\{1,2,\ldots, A\}$ is the finite action space, $P = [P_1,\ldots,P_H]$ with $P_h:\mathcal{S}\times\mathcal{A}\rightarrow\Delta(\mathcal{S})$ is the transition function, $R = [R_1,\ldots,R_H]$ with $R_h:\Sc\times\Ac\rightarrow[0,1]$ is the reward function, and $s_1$ is the deterministic initial state.\\ 

A deterministic (and non-stationary) policy  $\pi=(\pi_1,\ldots,\pi_H)$ is a sequence of functions, where each $\pi_h:\Sc\rightarrow\Ac$ defines the action to take at each state. The RL agent interacts with the environment across $K$ episodes giving us $T=KH$ steps in total. In episode $k$, the agent start with initial state $s_1^k = s_1$ and then follows policy $\pi^k$, thus inducing trajectory  $s_1^k,a_1^k,r_1^k,s_2^k,a_2^k,r_2^k,\ldots,s_H^k,a_h^k,r_H^k$.\\

For any timestep $h$ and state-action pair $(s,a)\in\Scal\times\Acal$, the Q-value function of policy $\pi$ is defined as $Q_h^\pi(s,a)=R_h(s,a)+\EE_\pi[\sum_{l=h}^H R_l(s_l,\pi_l(s_l)|s,a)]$ and the state-value function is defined as $V_h^\pi(s)=Q_h^\pi(s,\pi_h(s))$. We use $\pi^*$ to denote the optimal policy. The optimal state-value function is defined as $V_h^*(s)\coloneqq V_h^{\pi^*}(s)=\max_\pi V_h^\pi(s)$ and the optimal Q-value function is defined as $Q_h^*(s,a)\coloneqq Q_h^{\pi^*}(s,a)=\max_\pi Q_h^\pi(s,a)$. Both $Q^\pi$ and $Q^*$ satisfy Bellman equations
$$Q_h^\pi(s,a)=R_h(s,a)+\EE_{s'\sim P_h(\cdot|s,a)}[V_{h+1}^\pi(s')]$$
$$Q_h^*(s,a)=R_h(s,a)+\EE_{s'\sim P_h(\cdot|s,a)}[V_{h+1}^*(s')]$$
where $V_{H+1}^\pi(s)=V_{H+1}^*(s)=0$ $\forall s$. Notice that by the bounded nature of the reward function, for any $(h,s,a)$, all functions $Q_h^*,V_h^*,Q_h^\pi,V_h^\pi$ are within the range $[0,H-h+1]$. Since we consider the time-inhomogeneous setting (reward and transition change with timestep $h$), we have subscript $h$ on policy and value functions, and later traverse over $(h,s,a)$ instead of $(s,a)$. 

\paragraph{Regret}
An RL algorithm is a random mapping from the history until the end of episode $k-1$ to policy $\pi^k$ at episode $k$. We use regret to evaluate the performance of the algorithm:
$${\rm \text{Reg}}(K) = \sum_{k=1}^{K} V_{1}^{*}(s_1) - V_{1}^{\pi^k}(s_1).$$

Regret ${\rm \text{Reg}}(K)$ is a random variable, and we bound it with high probability $1-\delta$. We emphasize that high-probability regret bound provides a stronger guarantee on each roll-out~\citep{seldin2013evaluation,lattimore2020bandit} and can be converted to the same order of expected regret bound $${\rm \text{E-Reg}}(K) = \E\left[\sum_{k=1}^{K} V_{1}^{*}(s_1) - V_{1}^{\pi^k}(s_1)\right]$$ 
by setting $\delta=1/T$. However, expected regret bound does not imply small variance for each run. Therefore it can violate the same order of high-probability regret bound. We also point out that both bounds hold for all MDP instances $M$ that have $S$ states, $A$ actions, horizon $H$, and bounded reward $R\in[0,1]$. In other words, we consider worst-case (frequentist) regret bound.

\paragraph{Empirical MDP} We define the number of visitation of $(s,a)$ pair at timestep $h$ until the end of episode $k-1$ as $n^k(h,s,a)=\sum_{l=1}^{k-1}\ind\{(s^l_h,a^l_h)=(s,a)\}$. We also construct empirical reward and empirical transition function as $\hat R^k_{h,s,a}=\frac{1}{n^k(h,s,a)+1}\sum_{l=1}^{k-1}\ind\{(s^l_h,a^l_h)=(s,a)\}r_h^l$ and $\hat P^k_{h,s,a}(s')=\frac{1}{n^k(h,s,a)+1}\sum_{l=1}^{k-1}\ind\{(s^l_h,a^l_h,s_{h+1}^l)=(s,a,s')\}.$
Finally, we use $\hat{M}^k=(H,\mathcal{S},\mathcal{A},\hat{P}^k,\hat{R}^k,s_1^k)$ to denote the empirical MDP. Notice that we have $n^k(h,s,a)+1$ in the denominator, and it is not standard. The reason we have that is due to the analysis between model-free view and model-based view in Section \ref{sec:algorithm main}. In the current form, $\hat P_{h,s,a}^k$ is no longer a valid probability function, and it is for ease of presentation. More formally, we can slightly augment the state space by adding one absorbing state for each level $h$ and let all $(h,s,a)$ transit to the absorbing states with remaining probability.  
\section{C-RLSVI Algorithm}\label{sec:algorithm main}
The major goal of this paper is to improve the regret bound of TS-based algorithms in the tabular setting. Different from using fixed bonus term in the optimism-in-face-of-uncertainty (OFU) approach, TS methods~\citep{agrawal2013thompson,abeille2017linear,russo2019worst,zanette2020frequentist} facilitate exploration by making large enough random perturbation so that optimism is obtained with at least a constant probability. However, the range of induced value function can easily grow unbounded and this forms a key obstacle in previous analysis \cite{russo2019worst}. To address this issue, we apply a common clipping technique in RL literature~\citep{azar2017minimax,zanette2020frequentist,yang2019reinforcement}.\\

We now formally introduce our algorithm C-RLSVI as shown in Algorithm \ref{alg: RLSVI}. C-RLSVI follows a similar approach as RLSVI in~\citet{russo2019worst}. The algorithm proceeds in episodes. In episode $k$, the agent first samples $Q_h^{\text{pri}}$ from prior $\mathcal N(0,\frac{\beta_k}{2}I)$ and adds random perturbation on the data \textbf{(lines 3-10)}, where $\mathcal{D}_{h}=\{(s_{h}^{l}, a_{h}^{l}, r_{h}^{l}, s_{h+1}^{l} ) : l <k \}$ for $h<H$ and $\mathcal{D}_{H}=\{(s_{H}^{l}, a_{H}^{l}, r_{H}^{l}, \emptyset ) : l <k \}$. The injection of Gaussian perturbation (noise) is essential for the purpose of exploration and we set $\beta_k=\noiseBeta$. Later we will see the magnitude of $\beta_k$ plays a crucial role in the regret bound and it is tuned to satisfy the optimism with a constant probability 
in Lemma~\ref{lem: Optimism Main}. Given history data, the agent further performs the following procedure from timestep $H$ back to timestep 1: (i) conduct regularized least square regression \textbf{(lines 13)}, where $\mathcal{L}(Q \mid Q', \mathcal{D})=
\sum_{(s,a,r,s')\in \mathcal{D}} (Q(s,a) - r- \max_{a'\in \Ac} Q'(s',a'))^2$, and (ii) clips the Q-value function to obtain $\overline Q_k$ \textbf{(lines 14-19)}. Finally, the clipped Q-value function $\overline Q_k$ is used to extract the greedy policy $\pi^k$ and the agent rolls out a trajectory with $\pi^k$ (\textbf{lines 21-22)}.\\

\begin{algorithm}[ht]
\begin{algorithmic}[1]
	\STATE \textbf{input:} variance $\beta_{k}$ and clipping threshold $\alpha_k$;
	\FOR{episode $k=1,2,\ldots,K$ }
		\FOR{timestep $h=1,2, \ldots, H$}
			\STATE Sample prior $Q^{\text{pri}}_h \sim \mathcal N(0, \frac{\beta_k}{2} I)$;
			\STATE $\dot{D}_{h} \leftarrow \{\}$;
			\FOR{$(s,a,r,s')\in \mathcal{D}_h$}
				\STATE Sample $w\sim \mathcal N(0,\beta_k/2)$;
				\STATE $\dot{\mathcal{D}}_h \leftarrow \dot{\mathcal{D}}_h \cup \{(s,a,r+w,s')\}$;
			\ENDFOR
		\ENDFOR
		\STATE Define terminal value $\overline Q_{H+1,k}(s,a)\leftarrow 0 \quad \forall s,a$;
		\FOR{timestep $h = H,H-1,\ldots , 1$}
			\STATE $\hat{Q}_{h}^k \leftarrow \argmin_{Q\in \mathbb{R}^{SA}}  \left[\mathcal{L}(Q \mid \overline Q_{h+1,k}, \dot{\mathcal{D}}_{h}) + \|Q-Q^{\text{pri}}_h \|_2^2\right]$;
			\STATE \textsl{(Clipping)} 
    		$\forall(s,a)$\\
    		\IF{$n^k(h,s,a) > \alpha_k$} 
    		\STATE $\overline{Q}_{h,k}(s,a) = \hat{Q}_h^k(s,a)$; 
    		\ELSE 
    		\STATE $\overline{Q}_{h,k}(s,a) = H-h+1$;
    		\ENDIF
		\ENDFOR
		\STATE Apply greedy policy ($\pi^k)$ with respect to $ (\overline{Q}_{1,k}, \ldots \overline{Q}_{H,k})$ throughout episode;
		\STATE Obtain trajectory $s_{1}^k,a^k_1,r^{k}_1,\ldots s^k_{H}, a^k_H, r^k_H$;
	\ENDFOR
\end{algorithmic}
\caption{\algo}\label{alg: RLSVI}
\end{algorithm}

C-RLSVI as presented is a model-free algorithm, which can be easily extended to more general setting and achieve computational efficiency~\citep{osband2016deep,zanette2020frequentist}. When the clipping does not happen, it also has an equivalent model-based interpretation~\citep{russo2019worst} by leveraging the equivalence between running Fitted Q-Iteration~\citep{geurts2006extremely, chen2019information} with batch data and using batch data to first build empirical MDP and then conducting planing. In our later analysis, we will utilize the following property (Eq \ref{eq: blr}) of Bayesian linear regression \citep{russo2019worst,osband2019deep} for \textbf{line 13}
\begin{align}
\label{eq: blr}
\hat Q_h^k(s,a)|\overline Q_{h+1,k}&\sim
\mathcal{N}\large(\hat R^k_{h,s,a}+\sum_{s'\in S}\hat P^k_{h,s,a}(s')\max_{a'\in\Acal} \overline Q_{h+1,k}(s',a'),\frac{\beta_k}{2(n^k(h,s,a)+1)}\large)\nonumber\\
&\sim\hat R^k_{h,s,a}+\sum_{s'\in S}\hat P^k_{h,s,a}(s')\max_{a'\in\Acal} \overline Q_{h+1,k}(s',a')+w^k_{h,s,a},
\end{align}
where the noise term $w^k\in\RR^{HSA}$ satisfies $w^k(h,s,a)\sim \mathcal{N}(0,\sigma_k^2(h,s,a))$ and $\sigma_k(h,s,a)=\frac{\beta_k}{2(n^k(h,s,a)+1)}$. In terms of notation, we denote $\overline V_{h,k}(s)=\max_a \overline Q_{h,k}(s,a)$.\\

Compared with RLSVI in~\citet{russo2019worst}, we introduce a clipping technique to handle the abnormal case in the Q-value function. C-RLSVI has simple one-phase clipping and the threshold $\alpha_k=\CthresholdSquared$ is designed to guarantee the boundness of the value function. Clipping is the key step that allows us to introduce new analysis as compared to~\citet{russo2019worst} and therefore obtain a high-probability regret bound. Similar to as discussed in~\citet{zanette2020frequentist}, we want to emphasize that clipping also hurts the optimism obtained by simply adding Gaussian noise. However, clipping only happens at an early stage of visiting every $(h,s,a)$ tuple. Intuitively, once $(h,s,a)$ is visited for a large number of times, its estimated Q-value will be rather accurate and concentrates around the true value (within $[0, H-h+1]$), which means clipping will not take place. Another effect caused by clipping is we have an optimistic Q-value function at the initial phase of exploration since $Q^*_h \le H-h+1$. However, this is not the crucial property that we gain from enforcing clipping. Although clipping slightly breaks the Bayesian interpretation of RLSVI~\citep{russo2019worst,osband2019deep}, it is easy to implement empirically and we will show it does not introduce a major term on regret bound.
\section{Main Result}\label{sec: main results main}
In this section, we present our main result: high-probability regret bound in Theorem~\ref{thm: high probability regret main}.

\begin{thm}\label{thm: high probability regret main}
For $0<\delta < 4\Phi(-\sqrt 2)$, \algo\, enjoys the following high-probability regret upper bound, with probability $1-\delta$,
\begin{align*}
    {\rm Reg}(K) = \tilde{\mathrm{O}}\left( H^2S\sqrt{AT}\right).
\end{align*}
\end{thm}
Theorem~\ref{thm: high probability regret main} shows that \algo\, matches the state-of-the-art TS-based method~\citep{agrawal2017optimistic}. Compared to the lower bound~\citep{jaksch2010near}, the result is at least off by $\sqrt{HS}$ factor. This additional factor of $\sqrt{HS}$ has eluded all the worst-case analyses of TS-based algorithms known to us in the tabular setting. This is similar to an extra $\sqrt{d}$ factor that appears in the worst-case upper bound analysis of TS for $d$-dimensional linear bandits~\citep{abeille2017linear}.\\

It is useful to compare our work with the following contemporaries in related directions.
\paragraph{Comparison with~\citet{russo2019worst}} 
Other than the notion of clipping (which only contributed to warm-up or burn-in term), the core of \algo\, is the same as RLSVI considered by~\citet{russo2019worst}. Their work presents significant insights about randomized value functions but the analysis does not extend to give high-probability regret bounds, and the latter requires a fresh analysis. Theorem~\ref{thm: high probability regret main} improves his worst-case expected regret bound $\tilde{\mathrm{O}}(H^{5/2}S^{3/2}\sqrt{AT})$ by $\sqrt{HS}$.

\paragraph{Comparison with~\citet{zanette2020frequentist}}
Very recently,~\citet{zanette2020frequentist} proposed frequentist regret analysis for a variant of RLSVI with linear function approximation and obtained high-probability regret bound of $\tilde{\mathrm{O}}\del{H^2d^2\sqrt{T}}$, where $d$ is the dimension of the low rank embedding of the MDP. While they present some interesting analytical insights which we use (see Section~\ref{sec: proof outline main}), directly converting their bound to tabular setting ($d\,\rightarrow \,SA$) gives us quite loose bound $\tilde{\mathrm{O}}\del{H^2S^2A^2\sqrt{T}}$.

\paragraph{Comparison with~\citet{azar2017minimax,jin2018q}} These OFU works guaranteeing optimism almost surely all the time are fundamentally different from RLSVI. However, they develop key technical ideas which are useful to our analysis, e.g. clipping estimated value functions and estimation error propagation techniques. Specifically, in~\citet{azar2017minimax,dann2017unifying,jin2018q}, the estimation error is decomposed as a recurrence. Since RLSVI is only optimistic with a constant probability (see Section~\ref{sec: proof outline main} for details), their techniques need to be substantially modified to be used in our analysis.
\section{Proof Outline}\label{sec: proof outline main}
In this section, we outline the proof of our main results, and the details are deferred to the appendix. The major technical flow is presented from Section~\ref{sec: regret as sum of estimation and pessimism main} onward. Before that, we present three technical prerequisites: (i) the total probability for the unperturbed estimated $\hat{M}^k$ to fall outside a confidence set is bounded; (ii) $\barV_{1,k}$ is an upper bound of the optimal value function $V^*_{1}$ with at least a constant probability at every episode; (iii) the clipping procedure ensures that $\barV_{h,k}$ is bounded with high probability\footnote{We drop/hide constants by appropriate use of $\gtrsim,\lesssim,\simeq$ in our mathematical relations. All the detailed analyses can be found in our appendix.}.

\paragraph{Notations} To avoid cluttering of mathematical expressions, we abridge our notations to exclude the reference to $(s,a)$ when it is clear from the context. Concise notations are used in the later analysis: $\Rhk\rightarrow \Rhkmain$, $\RHathk\rightarrow \hat R_h^k$, $\Phk \rightarrow \Phkmain$, $\PHathk \rightarrow \PHathkmain$, $\nhk \rightarrow \nhkmain$,  $\noisehk \rightarrow\noisehkmain$.

\paragraph{High probability confidence set}In Definition~\ref{def:confidence set main}, $\mathcal{M}^k$ represents a set of MDPs, such that the total estimation error with respect to the true MDP is bounded. 
\begin{definition}[Confidence set]\label{def:confidence set main}
\begin{align*}
&\mathcal{M}^k =\bigg\{ (H, \Sc, \Ac, P', R', s_1) : \forall(h,s,a),\, \left|\R'-\R + \langle P'_{h,s,a}-P_{h,s,a},V^*_{h+1}\rangle \right|\,\,\, 
\leq \sqrt{e_k(h,s,a)}  \bigg\}
\end{align*}
where we set 
\begin{equation}\label{eq: confidence interval main}
\sqrt{e_k(h,s,a)} = 
  H\sqrt{ \frac{ \log\left( 2HSA k  \right) }{n^k(h,s,a)+1}}. 
\end{equation}
\end{definition}

Through an application of Hoeffding's inequality~\citep{jaksch2010near,osband2013more}, it is shown via Lemma~\ref{lem: confidence interval lemma main} that the empirical MDP does not often fall outside confidence set $\mathcal{M}^k$. This ensures {\sf exploitation}, i.e., the algorithm's confidence in the estimates for a certain $(h,s,a)$ tuple grows as it visits that tuple many numbers of times.

\begin{lemma}\label{lem: confidence interval lemma main}
    $\sum_{k=1}^{\infty} \mathbb{P}\left(\hat{M}^k \notin \mathcal{M}^k \right) \leq 2006HSA$.
\end{lemma} 

\paragraph{Bounded Q-function estimates}It is important to note the pseudo-noise used by \algo\, has both exploratory ({\sf optimism}) behavior and corrupting effect on the estimated value function. Since the Gaussian noise is unbounded, the clipping procedure (lines 14-19 in Algorithm~\ref{alg: RLSVI}) avoids propagation of unreasonable estimates of the value function, especially for the tuples $(h,s,a)$ which have low visit counts. This saves from low rewarding states to be misidentified as high rewarding ones (or vice-versa). Intuitively, the clipping threshold $\alpha_k$ is set such that the noise variance ($\sigma_k(h,s,a)=\frac{\beta_k}{2(n^k(h,s,a)+1)}$) drops below a numerical constant and hence limiting the effect of noise on the estimated value functions. This idea is stated in Lemma~\ref{lem: est Q function bounded main}, where we claim the estimated Q-value function is bounded for all $(h,s,a)$. 

\begin{lemma}[(Informal) Bound on the estimated Q-value function]\label{lem: est Q function bounded main}
Under some good event, for the clipped Q-value function $\barQ_k$ defined in Algorithm~\ref{alg: RLSVI}, we have $|(\barQ^{\polbarMk}_{h,k} - Q^*_{h})(s,a)| \leq H-h +1.$
\end{lemma}
See Appendix~\ref{sec: concentration of events} for the precise definition of good event and a full proof. Lemma~\ref{lem: est Q function bounded main} is striking since it suggests that randomized value function needs to be clipped only for constant (i.e. independent of $T$) number of times to be well-behaved.

\paragraph{Optimism}The event when none of the rounds in episode $k$ need to be clipped is denoted by $\Ethresk \coloneqq \{ \cap_{h\in[H]}(\nhkmain \geq \alpha_k)\}.$
Due to the randomness in the environment, there is a possibility that a learning algorithm may get stuck on ``bad'' states, i.e. not visiting the ``good'' $(h,s,a)$ enough or it grossly underestimates the value function of some states and as result avoid transitioning to those state. Effective {\sf exploration} is required to avoid these scenarios. To enable correction of faulty estimates, most RL exploration algorithms maintain optimistic estimates almost surely. However, when using randomized value functions, \algo\, does not always guarantee optimism. In Lemma~\ref{lem: Optimism Main}, we show that \algo\, samples an optimistic value function estimate with at least a constant probability for any $k$. We emphasize that such difference is fundamental.

\begin{lemma}\label{lem: Optimism Main}
When conditioned on $\mathcal{H}_{H}^{k-1}$, $$\mathbb{P}\left(\barV_{1,k}(s^k_1) \geq V^*_{1}(s^k_1)\,| \,\mathcal{G}_k\right) \geq \OptPr.$$ 
Here $\Phi(\cdot)$ is the CDF of $\mathcal N(0,1)$ distribution and $\mathcal{H}_{H}^{k-1}$ is all the history of the past observations made by \algo\, till the end of the episode $k-1$. We use $\mathcal{G}_k$ to denote the good intersection event of $\hat M^k\in\mathcal M^k$ and bounded noises and values (the specific definition of good event $\mathcal{G}_k$ can be found in Appendix~\ref{sec: notations}).
\end{lemma}
Lemma~\ref{lem: Optimism Main} is adapted from \citet{zanette2020frequentist,russo2019worst}, and we reproduce the proof in Appendix \ref{sec: optimism} for completeness.\\

Now, we are in a position to simplify the regret expression with high probability as
\begin{align}
&{\rm Reg}(K)\leq~ 
\sum_{k=1}^{K} \ind\{\mathcal{G}_k\}\left( V_{1}^{*}  - V_{1}^{\pi^k}\right)(s_1^k)  +H\underbrace{\mathbb{P}(\hat{M}^k\notin\mathcal{M}^k)}_{\text{Lemma~\ref{lem: confidence interval lemma main}}},
\label{eq: regret decom main text}
\end{align}
where we use Lemma~\ref{lem: confidence interval lemma main} to show that for any $(h,s,a)$, the edge case that the estimated MDP lies outside the confidence set is a transient term (independent of $T$). The proof of Eq (\ref{eq: regret decom main text}) is deferred to Appendix~\ref{sec: apx_d}.\\

We also define $\tilde{w}^k(h,s,a)$ as an independent sample from the same distribution $\mathcal{N}(0,\sigma^2_k(h,s,a))$ conditioned on the history of the algorithm till the last episode. Armed with the necessary tools, over the next few subsections we sketch the proof outline of our main results. All subsequent discussions are under good event $\mathcal{G}_k$.

\subsection{Regret as Sum of Estimation and Pessimism}\label{sec: regret as sum of estimation and pessimism main}
Now the regret over $K$ episodes of the algorithm decomposes as
\begin{equation}\label{eq: regret decomp pes est terms}
    \sum^K_{k=1}\big(\underbrace{(V_{1}^{*} -\barV_{1,k})(s^k_1)}_{\text{Pessimism}} + \underbrace{\barV_{1,k}(s^k_1) -  V_{1}^{\pi^k}(s^k_1)}_{\text{Estimation}}\big).
\end{equation}
In OFU-style analysis, the pessimism term is non-positive and insignificant \cite{azar2017minimax,jin2018q}. In TS-based analysis, the pessimism term usually has zero expectation or can be upper bounded by the estimation term~\cite{osband2013more,osband2017posterior,agrawal2017optimistic,russo2019worst}. Therefore, the pessimism term is usually relaxed to zero or reduced to the estimation term, and the estimation term can be bounded separately. Our analysis proceeds quite differently. In Section~\ref{sec: pessimism in terms of estimation}, we show how the pessimism term is decomposed to terms that are related to the algorithm's trajectory (estimation term and pessimism correction term). In Section~\ref{sec: estimation bound main} and Section~\ref{sec: pessimism correction bound main}, we show how to bound these two terms through two independent recurrences. Finally, in Section~\ref{sec: final regret bound}, we reorganize the regret expression whose individual terms can be bounded easily by known concentration results.

\subsection{Pessimism in Terms of Estimation}\label{sec: pessimism in terms of estimation}

In this section we present Lemma~\ref{lem: pessimism decomp main}, where the pessimism term is bounded in terms of the estimation term and a correction term $(V_{1}^{\pi^k}-\underV_{1,k})(s^k_1)$ that will be defined later. This correction term is further handled in Section~\ref{sec: pessimism correction bound main}. While the essence of Lemma~\ref{lem: pessimism decomp main} is similar to that given by~\citet{zanette2020frequentist}, there are key differences: we need to additionally bound the correction term; the nature of the recurrence relations for the pessimism and estimation terms necessitates a distinct solution, hence leading to different order dependence in regret bound. In all, this allows us to obtain stronger regret bounds as compared to~\citet{zanette2020frequentist}.

\begin{lemma}\label{lem: pessimism decomp main}
Under the event $\mathcal{G}_k$,
\begin{align}\label{eq:lem9-1-a main}
&~(V^*_{1}-\barV_{1,k})(s^k_1)\lesssim~  (\barV_{1,k}-V^{\pi^k}_{1})(s^k_1)+(V^{\pi^k}_{1}-\underV_{1,k})(s^k_1) + \MDSfkind{1},
\end{align}
where $\MDSfkind{1}$ is a martingale difference sequence (MDS).
\end{lemma}
The detailed proof can be found in Appedix~\ref{sec: Bounds on Pessimism}, while we present an informal proof sketch here. The general strategy in bounding $V^*_{1}(s^k_1) - \barV_{1,k}(s^k_1)$ is that we find an upper bounding estimate of $V^*_{1}(s^k_1)$ and a lower bounding estimate of $\barV_{1,k}(s^k_1)$, and show that the difference of these two estimates converge. 

We define $\tilV_{1,k}$ to be the value function obtained when running Algorithm \ref{alg: RLSVI} with random noise $\tilde w$. Since $\tilde w$ is i.i.d. with $w$, we immediately have that $\tilV_{1,k}$ is i.i.d. to $\barV_{1,k}$. Then we introduce the following optimization program:
\begin{align*}
    & \underset{w^k_{\text{ptb}}\, \in\,\mathbb{R}^{HSA}}{\min} V^{ w^k_{\text{ptb}}}_{1,k}(s_1^k)\nonumber\\
     s.t.& \quad | w^k_{\text{ptb}}(h,s,a)| \leq \gamma_k(h,s,a)\quad\forall\, h,s,a,
\end{align*}
where $V^{ w^k_{\text{ptb}}}_{1,k}(s_1^k)$ is analogous to $\overline V_{1,k}$ obtained from Algorithm \ref{alg: RLSVI} but with optimization variable $w^k_{\text{ptb}}$ in place of $w^k$. We use $\underline w^k$ to denote the solution of the optimization program and $\underline V$ to be the minimum. This ensures $\underV_{1,k} \leq \barV_{1,k}$ and $\underV_{1,k} \leq \tilV_{1,k}$.
Thus the pessimism term is now given by
\begin{align}
    (V^*_{1} - \barV_{1,k})(s^k_1) \leq (V^*_{1} - \underV_{1,k})(s^k_1).\label{eq: lemma pes 1 main}
\end{align}
Define event $\EoptTwohk \coloneqq \{ \tilV_{1,k}(s^k_1) \geq V^*_{1}(s^k_1) \}$, $\tilde {\mathcal G}_k$ to be a similar event as $\mathcal G_k$, and use $\mathbb{E}_{\tilde w}\sbr{\cdot}$ to denote the expectation over the pseudo-noise $\tilde{w}$. Since $V^*_{1}(s^k_1)$ does not depend on $\tilde{w}$, we get $V^*_{1}(s^k_1) \leq \mathbb{E}_{\tilde{w}|\EoptTwohk,\tilde {\mathcal G}_k}\sbr{ \tilV_{1,k}(s^k_1)}$. We can further upper bound Eq~(\ref{eq: lemma pes 1 main}) by
\begin{align}
    (V^*_{1} - \underV_{1,k})(s^k_1) \leq \mathbb{E}_{\tilde{w}|\EoptTwohk,\tilde {\mathcal G}_k}[ (\tilV_{1,k}-\underV_{1,k})(s^k_1)]\label{eq: pessimism 3 old main}.
\end{align}
Thus, we are able to relate pessimism to quantities which only depend on the algorithm's trajectory. Further we upper bound the expectation over marginal distribution $\mathbb{E}_{\tilde{w}|\EoptTwohk,\tilde {\mathcal G}_k}[\cdot]$ by $\mathbb{E}_{\tilde{w}|\tilde {\mathcal G}_k}[\cdot]$. This is possible because we are taking expectation of non-negative entities. Moreover, we can show:
\begin{align}
\mathbb{E}_{\tilde{w}|\tilde {\mathcal G}_k}[ (\tilV_{1,k}-\underV_{1,k})(s^k_1)]
\simeq& ~\MDSfkind{1} + \barV_{1,k}(s^k_1) -  \underV_{1,k}(s^k_1),
\label{eq: pessimism 4 main}
\end{align}
Now consider
\begin{align}\label{eq: pessimism decom main}
(\barV_{1,k} -  \underV_{1,k})(s^k_1)
=& \underbrace{(\barV_{1,k} -  V^{\pi^k}_{1})(s^k_1)}_{\text{Estimation term}} + \underbrace{(V^{\pi^k}_{1} -  \underV_{1,k})(s^k_1)}_{\text{Correction term}}.
\end{align}
In Eq~(\ref{eq: pessimism decom main}), the estimation term is decomposed further in Section~\ref{sec: estimation bound main}. The correction term is simplified in Section~\ref{sec: pessimism correction bound main}.

\subsection{Bounds on Estimation Term}\label{sec: estimation bound main}
In this section we show the bound on the estimation term. Under the high probability good event $\mathcal{G}_k$, we show decomposition for the estimation term $(\barV_{h,k} -  V^{\pi^k}_{h,k})(s^k_h)$ holds with high probability. By the property of Bayesian linear regression and the Bellman equation, we get
\begin{align}
&(\barV_{h,k} -  V^{\pi^k}_{h})(s^k_h) =\ind\{\Ethreshk\}(\underbrace{\langle\PHathkmain - \Phkmain ,\barV_{h+1,k}\rangle}_{(1)}+ \underbrace{\langle \Phkmain, \barV_{h+1,k} -  V^{\pi^k}_{h+1}\rangle}_{(1')}+
\RHathkmain-\Rhkmain+\noisehkmain)+H \underbrace{\ind\{\mathcal{E}^{\text{th}\;\complement}_k\}}_{\text{Warm-up term}}.\label{eq: estimation bound main 1}
\end{align}
We first decompose Term (1) as
\begin{align}\label{eq: estimation bound main 5}
(1) = \underbrace{\langle\PHathkmain - \Phkmain ,V^*_{h+1}\rangle}_{(2)}+\underbrace{\langle\PHathkmain- \Phkmain,\barV_{h+1,k}-V^*_{h+1}\rangle}_{(3)}.
\end{align}
Term (2) represents the error in estimating the transition probability for the optimal value function $V^*_{h}$, while Term (3) is an offset term. The total estimation error, $\EstErrorhk\,:=\, \envert{\text{Term }(2) + \RHathkmain-\Rhkmain}$ is easy to bound since the empirical MDP $\hat M^k$ lies in the confidence set (Eq~(\ref{eq: confidence interval main})). Then we discuss how to bound Term (3). Unlike OFU-styled analysis, here we do not have optimism almost surely. Therefore we cannot simply relax $\barV_{h+1,k}-V^*_{h+1}$ to $\barV_{h+1,k}-V^{\pi^k}_{h+1}$ and form the recurrence. Instead, we will apply ($L_1,L_\infty$) Cauchy-Schwarz inequality to separate the deviation of transition function estimation and the deviation of value function estimation, and then further bound these two deviation terms. Noticing that $\overline V_{h+1,k}-V^*_{h+1}$ might be unbounded, we use Lemma~\ref{lem: est Q function bounded main} to assert that $\|V_{h+1}^*-\barV_{h+1}\|_\infty \le H$ under event $\Gcal_k$. With the boundedness of the deviation of value function estimation, it suffices to bound the remaining $\|\PHath- \Ph\|_1$ term. Proving an $L_1$ concentration
bound for multinomial distribution with careful application of the Hoeffding's inequality shows
$$ 
\|\PHathkmain- \Phkmain\|_1\le 4\sqrt{\frac{SL}{\nhkmain+ 1}},
$$
where $L=\decompCL$. In Eq~(\ref{eq: estimation bound main 1}), we also decomposes Term (1') to a sum of the next-state estimation and a MDS.\\

Clubbing all the terms starting from Eq~(\ref{eq: estimation bound main 1}), with high probability, the upper bound on estimation is given by
\begin{align}
(\barV_{h,k}-V^{\pi^k}_{h} )(s^k_h)
\lesssim& ~  \ind\{\Ethreshk\}\Big(\underbrace{(\barV_{h+1,k}- V^{\pi^k}_{h+1} )(s^k_{h+1})}_{\text{Next-state estimation}}+ \EstErrorhk + \noisehkmain+ \MDSbk + 4H\sqrt{\frac{SL}{\nhkmain + 1}}\Big)+H \ind\{\mathcal{E}^{\text{th}\;\complement}_k\},\label{eq: estimation decomp main 50}
\end{align}
where $\MDSbk$ is a Martingale difference sequence (MDS). Thus, via Eq~(\ref{eq: estimation decomp main 50}) we are able decompose estimation term in terms of total estimation error, next-step estimation, pseudo-noise, a MDS term, and a $\tilde{\mathrm{O}}\del{\sqrt{1/n^k(h)}}$ term. From the form of Eq~(\ref{eq: estimation decomp main 50}), we can see that it forms a recurrence. Due to this style of proof, our Theorem~\ref{thm: high probability regret main} is $\sqrt{HS}$ superior than the previous state-of-art result~\citep{russo2019worst}, and we are able to provide a high probability regret bound instead of just the expected regret bound.

\subsection{Bounds on Pessimism Correction}\label{sec: pessimism correction bound main}
In this section, we give the decomposition of the pessimism correction term $(V^{\pi^k}_{h} - \underV_{h,k})(s^k_h)$. Shifting from $\overline V_k$ to $\underline V_k$ and re-tracing the steps of Section~\ref{sec: estimation bound main}, with high probability, it follows 
\begin{align}
&~(V^{\pi^k}_{h} - \underV_{h,k})(s^k_h)\lesssim\ind\{\Ethreshk\}\Big(\underbrace{(V^{\pi^k}_{h+1} - \underV_{h+1,k})(s^k_{h+1})}_{\text{Next-state pessimism correction}}+\,\,\EstErrorhk + \envert{\noisehUkmain}+\MDSck+ 4H\sqrt{\frac{SL}{n^k(h) + 1}}\Big)+H \ind\{\mathcal{E}^{\text{th}\;\complement}_k\}.
\label{eq: pessimism correction decomp main}
\end{align}
The decomposition Eq~(\ref{eq: pessimism correction decomp main}) also forms a recurrence. The recurrences due to Eq~(\ref{eq: estimation decomp main 50}) and Eq~(\ref{eq: pessimism correction decomp main}) are later solved in Section~\ref{sec: final regret bound}. 
\subsection{Final High-Probability Regret Bound}\label{sec: final regret bound}
To solve the recurrences of Eq~(\ref{eq: estimation decomp main 50}) and Eq~(\ref{eq: pessimism correction decomp main}), we keep unrolling these two inequalities from $h=1$ to $h=H$. Then with high probability, we get 
\begin{align*}
{\rm Reg}(K) \lesssim &\sum_{k=1}^{K}\sum_{h=1}^{H} 
\left(\prod_{h'=1}^h\ind\{\mathcal{E}_{h',k}^{\text {th}}\}\Big(\envert{\EstErrorhk} + \envert{\noisehUkmain}+\noisehkmain + \MDSck  + \MDSbk\right.\\
&\left.+ 4H\sqrt{\frac{SL}{n^k(h) + 1}}\Big)+H \ind\{\mathcal{E}^{\text{th}\;\complement}_k\}  \right)+\sum_{k=1}^K\MDSfkind{1} .\label{eq: final regret decomp main}
\end{align*}
Bounds of individual terms in the above equation are given in Appendix~\ref{sec: bounds on individual terms}, and here we only show the order dependence.\\

The maximum estimation error that can occur at any round is limited by the size of the confidence set Eq~(\ref{eq: confidence interval main}). Lemma~\ref{lem: estimation error} sums up the confidence set sizes across the $h$ and $k$ to obtain $\sum_{k=1}^K\sum_{h=1}^{H} \envert{\EstErrorhk} = \tilde{\mathrm{O}}(\sqrt{H^3SAT})$. In Lemma~\ref{lem: MDS concentration}, we use Azuma-Hoeffding inequality to bound the summations of the martingale difference sequences with high probability by $\tilde{\mathrm{O}}(H\sqrt{T})$.
The pseudo-noise $\sum_{k=1}^K\sum_{h=1}^{H}\noisehkmain$ and the related term $\sum_{k=1}^K\sum_{h=1}^{H}\noisehUkmain$ are bounded in Lemma~\ref{lem: estimation non random noise} with high probability by $\tilde{\mathrm{O}}(H^2S\sqrt{AT})$.
Similarly, we have $\sum_{k=1}^K\sum_{h=1}^{H}\sqrt{\frac{SL}{n^k(h) + 1}} = \tilde{\mathrm{O}}(H^2S\sqrt{AT})$ from
Lemma~\ref{lem: estimation non random noise}. Finally, Lemma~\ref{lem: warmup bound} shows that the warm-up term due to clipping is independent on $T$. Putting all these together yields the high-probability regret bound of Theorem~\ref{thm: high probability regret main}.

\section{Discussions and Conclusions}
In this work, we provide a sharper regret analysis for a variant of RLSVI and advance our understanding of TS-based algorithms. Compared with the lower bound, the looseness mainly comes from the magnitude of the noise term in random perturbation, which is delicately tuned for obtaining optimism with constant probability. Specifically, the magnitude of $\beta_k$ is $\tilde{\mathrm{O}}(\sqrt{HS})$ larger than sharpest bonus term~\citep{azar2017minimax}, which leads to an additional $\tilde{\mathrm{O}}(\sqrt{HS})$ dependence. Naively using a smaller noise term will affect optimism, thus breaking the analysis. Another obstacle to obtaining $\tilde{\mathrm{O}}(\sqrt{S})$ results is attributed to the bound on Term (3) of Eq~(\ref{eq: estimation bound main 5}). Regarding the dependence on the horizon, one $\mathrm{O}(\sqrt{H})$ improvement may be achieved by applying the law of total variance type of analysis in~\citep{azar2017minimax}. The future direction of this work includes bridging the gap in the regret bounds and the extension of our results to the time-homogeneous setting.

\section{Acknowledgements}
We gratefully thank the constructive comments and discussions from Chao Qin, Zhihan Xiong, and Anonymous Reviewers.
\bibliographystyle{aaai21.bst}
\bibliography{ref.bib}

\onecolumn
\appendix
\section{Notations, Constants and Definition}\label{sec: notations}


\subsection{Notation Table}
\begin{longtable}[H]{l c l }
\caption{Notation table}
\label{tab: notation}
\\\hline 
\textbf{Symbol} & & \textbf{Explanation}
\\
\hline
$\mathcal{S}$& & The state space
\\
$\mathcal{A}$& & The action space
\\
$S$& & Size of state space
\\
$A$& & Size of action space
\\
$H$& & The length of horizon
\\
$K$& & The total number of episodes
\\
$T$ & & The total number of steps across all episodes
\\
$\pi^k$& & The greedy policy obtained in the Algorithm \ref{alg: RLSVI} at episode $k$, $\pi^k=\{\pi^k_1,\cdots,\pi^k_H\}$
\\
$\pi^*$& & The optimal policy of the true MDP\\
$(s^k_h,a^k_h)$& & The state-action pair at timestep $h$ in episode $k$ \\
$(s^k_h,a^k_h,r^k_h)$& & The tuple representing state-action pair and the corresponding reward\\ && at timestep $h$ in episode $k$ \\
$\History$ && $ \{ (s^j_l,a^j_l,r^j_l):\text{if } j<k \text{ then}\,
h\leq H,\text{ else if }\,j=k\,\text{ then }\,l\leq h\}$
\\
&& The history (algorithm trajectory) till timestep $h$ of the episode $k$.
\\
$\HistoryN$ && $\mathcal{H}^k_{h}\,\bigcup\, \left\{  w^k(l,s,a):l\in[H],s\in\mathcal{S},a\in\mathcal{A}\right\}$\\
&& The union of the history (algorithm trajectory) til timestep $h$ in episode $k$ \\
&&and the pseudo-noise of all timesteps in episode $k$\\ 
$n^k(h,s,a)$& &$\sum_{l=1}^{k-1}\ind\{(s^l_h,a^l_h)=(s,a)\}$ \\ & &The number of visits to state-action pair $(s,a)$ in timestep $h$ upto episode $k$
\\
$\Phk$& &The transition distribution for the state action pair $(s^k_h,a^k_h)$\\ 

$\Rhk$& &The reward distribution for the state action pair $(s^k_h,a^k_h)$
\\
$P_{h,s,a}$& &The transition distribution for the state action pair $(s,a)$ at timestep $h$\\ 
$R_{h,s,a}$& &The reward distribution for the state action pair $(s,a)$ at timestep $h$\\
\\
$\PHathk$& &The estimated transition distribution for the state action pair $(s^k_h,a^k_h)$
\\
$\RHathk$& &The estimated reward distribution for the state action pair $(s^k_h,a^k_h)$
\\
$\mathcal{M}^k$ && The confidence set around the true MDP
\\
$\noisehk$ && The pseudo-noise used for exploration
\\
$\tilde{w}^k_{h,s^k_h,s^k_h}$ && The independently pseudo-noise sample, conditioned on history till epsiode $k-1$
\\
$\hat{M}^k$& & $(H,\mathcal{S},\mathcal{A},\hat{P}^k,\hat{R}^k,s^k_1)$\\&& The estimated MDP without perturbation in data in episode $k$
\\
$\barM^k$& & $(H,\mathcal{S},\mathcal{A},\hat{P}^k,\hat{R}^k+w^k,s^k_1)$\\&& The estimated MDP with perturbed data in episode $k$
\\
$V^*_{h}$& &The optimal value function under true MDP on the sub-episodes \\& & consisting of the timesteps $\{h,\cdots,H\}$ 
\\
$V^{\pi^k}_{h}$& &The state-value function of $\pi^k$ evaluated on the true MDP on the \\ & &sub-episodes consisting of the timesteps $\{h,\cdots,H\}$ \\
$\barV_{h,k}$& &The state-value function calculated in Algorithm \ref{alg: RLSVI}  with noise $w^k$\\
$\barQ_{h,k}$& &The Q-value function calculated in Algorithm \ref{alg: RLSVI}  with noise $w^k$\\
$\tilde{M}^k,\,\tilde{V}_{1,k},\,\tilde{w}_{h,s,a}^k$ &&Refer to Definition~\ref{def: tilde V}
\\
$\underline{M}^k,\,\underline{V}_{1,k},\,\underline{w}_{h,s,a}^k$ &&Refer to Definition~\ref{def: under V}
\\
$\overline \delta_{h,k}(s_h^k)$& &$ V^*_h(s_h^k) - \barV_{h,k}(s_h^k)$\\
$\underline \delta_{h,k}(s_h^k)$& &$V^*_h(s_h^k) - \underV_{h,k}(s_h^k)$\\
$\overline \delta^{\polbarM^k}_{h,k}(s_h^k)$& &$\barV_{h,k}(s_h^k) - V^{\pi^k}_{h,k}(s_h^k)$
\\
$\underline \delta^{\polbarM^k}_{h,k}(s_h^k)$& &$V^{\pi^k}_{h,k}(s_h^k)- \underV_{h,k}(s_h^k)$
\\
$\RDiffhk$& &$\RHathk-\Rhk$ \\
\\
$\PDiffhk$& &$\langle\PHathk- \Phk,V^*_{h+1}\rangle$ 
\\
$C$& &$\frac{1}{\OptPr}$
\\
$L$ & &$\decompCL$  
\\
$\sqrt{\alpha_k}$ && $\Cthreshold$
\\
$\sigma^2_k(h,s,a)$ && $\frac{\beta_k}{2(n^k(h,s,a) + 1)}=\noiseSigma$
\\
$\gamma_k(h,s,a)$ && $\sqrt{\sigma^2_k(h,s,a)L}$
\\
$\sqrt{e_{k}(h,s,a)}$ && 
  $H\sqrt{ \frac{ \log\left( 2HSA k  \right) }{n^k(h,s,a)+1}}$
\\
$\beta_k$ & &$\noiseBeta$
\\
$\MDSbk$ && Refer to Appendix~\ref{sec: MDS}
\\
$\MDSck$ && Refer to Appendix~\ref{sec: MDS}
\\
$\MDSfind{1}$ && Refer to Appendix~\ref{sec: MDS}\\
$\Econf$ &&$ \left\{ \hat{M}^k \in \mathcal{M}^k \right\}$\\
$\mathcal{E}^{w}_{h,k}$ &&$ \left\{|w^k(h,s,a)| \leq \gamma_k(h,s,a),\forall (s,a)\right\}$\\
$\mathcal{E}^w_k$ &&$ \left\{\underset{h\in[H]}{\cap}\left( \mathcal{E}^{w}_{h,k} \right)\right\}$\\
$\mathcal{E}^{\overline{Q}^{\polbarMk}}_{h,k}$ &&$ \left\{ |(\overline{Q}^{\polbarMk}_{h,k} - Q^*_{h})(s,a)| \leq H -h+1,\, \forall (s,a) \right\}$\\
$\mathcal{E}^{\overline{Q}^{\polbarMk}}_k$ &&$ \left\{\underset{h\in[H]}{\cap}\left( \mathcal{E}^{\overline{Q}^{\polbarMk}}_{h,k} \right)\right\}$\\
$\overline{\mathcal{E}}_k$ &&$ \left\{ \mathcal{E}^{w}_{k} \cap \mathcal{E}^{\overline{Q}^{\polbarMk}}_{k}\right\}$\\
$\mathcal{E}^{\tilde{w}}_{h,k}$ &&$ \left\{ |\tilde{w}^k(h, s,a)| \leq \gamma_k(h,s,a)), \forall (s,a)\right\}$\\
$\mathcal{E}^{\tilde{w}}_k$ &&$ \left\{\underset{h\in[H]}{\cap}\left( \mathcal{E}^{\tilde{w}}_{h,k} \right)\right\}$\\
$\mathcal{E}^{\tilde{Q}^{\polbarMk}}_{h,k}$ &&$ \left\{ |(\tilde{Q}^{\polbarMk}_{h,k} - Q^*_{h})(s,a)| \leq H -h+1,\, \forall (s,a) \right\}$\\
$\mathcal{E}^{\tilde{Q}^{\polbarMk}}_k$ &&$ \left\{\underset{h\in[H]}{\cap}\left( \mathcal{E}^{\tilde{Q}^{\polbarMk}}_{h,k} \right)\right\}$\\
$\tilde{\mathcal{E}}_k$ &&$ \left\{ \mathcal{E}^{\tilde{w}}_{k} \cap \mathcal{E}^{\tilde{Q}^{\polbarMk}}_{k}\right\}$\\
$\Ethreshk$ &&$ \left\{ n^k(h,s^k_h,a^k_h) \geq \alpha_k \right\}$\\
$\Ethresk$ && $\left\{\underset{h\in[H]}{\cap} \Ethreshk\right\}$\\
$\mathcal{G}_k$ &&$ \left\{\overline{\mathcal{E}}_{k}\cap \Econf \right\}$\\
$\tilde{\mathcal{G}}_k$ &&$ \left\{\tilde{\mathcal{E}}_{k}\cap \Econf \right\}$\\
$\Eopt$ &&$ \left\{ \overline{V}_{1,k}(s^k_1) \geq V^*_{1}(s^k_1) \right\}$\\
$\EoptTwohk$ && $\left\{ \tilde{V}_{1,k}(s^k_1) \geq V^*_{1}(s^k_1) \right\}$\\\hline
\end{longtable}

\subsection{Definitions of Synthetic Quantities}
In this section we define some synthetic quantities required for analysis.

\begin{definition}[$\tilde V_{h,k}$]\label{def: tilde V}
Given history $\mathcal{H}^{k-1}_H$, and $\hat P^k$ and $\hat R^k$ defined in empirical MDP $\overline {M}^k=(H,\mathcal{S},\mathcal{A},\hat{P}^k,\hat{R}^k,s_1^k)$, we define independent Gaussian noise term $\tilde{w}^k(h,s,a)| \mathcal{H}^{k-1}_H\,\sim\,\mathcal{N}(0,\sigma^2_k(h,s,a))$, perturbed MDP $\tilde{M}^k=(H,\mathcal{S},\mathcal{A},\hat{P}^k,\hat{R}^k+\tilde{w}^k,s_1^k)$, and $\tilV_{h,k}$ to be the value function obtained by running Algorithm \ref{alg: RLSVI} with random noise $\tilde w^k$.

Notice that $\tilde w^k$ can be different from the realized noise term $w^k$ sampled in the Algorithm~\ref{alg: RLSVI}. They are two independent samples form the same Gaussian distribution. Therefore, conditioned on the history $\mathcal{H}^{k-1}_H$, $\tilde{M}^k$ has the same marginal distribution as $\barM^k$, but is statistically independent of the policy $\pi^k$ selected by \algo. 
\end{definition}

\begin{definition}[$\underline V_{1,k}$]\label{def: under V}
Similar as in Definition~\ref{def: tilde V}, given history $\mathcal{H}^{k-1}_H$ and any fixed noise $ w_{\text{ptb}}^k\in\mathbb{R}^{HSA}$, we define a perturbed MDP ${M}^k_{\text{ptb}}=(H,\mathcal{S},\mathcal{A},\hat{P}^k,\hat{R}^k+w^k_{\text{ptb}},s_1^k)$ and $V^{w_{\text{ptb}}^k}_{h,k}$ to be the value function obtained by running Algorithm \ref{alg: RLSVI} with random noise $w_{\text{ptb}}^k$.

Let $\underline w^k$ be the solution of following optimization program
\begin{align*}
    & \underset{ w^k_{\text{ptb}}\, \in\,\mathbb{R}^{HSA}}{\min} V^{ w^k_{\text{ptb}}}_{1,k}(s_1^k)\nonumber\\
     s.t.& \quad |w^k_{\text{ptb}}(h,s,a)| \leq \gamma_k(h,s,a)\quad\forall\, h,s,a.
\end{align*}
We also use $\underV_{h,k}$ to denote the minimum of the optimization program (i.e., value function $V^{\underline w^k}_{h,k}$) and define MDP $\underline{M}^k=(H,\mathcal{S},\mathcal{A},\hat{P}^k,\hat{R}^k+\underline{w}^k,s_1^k)$. Then we get that $\underV_{1,k} \le V^{w^k_{\text{ptb}}}_{1,k}$ for any $|w^k_{\text{ptb}}| \leq \gamma_k$.
\end{definition}

\begin{definition}[Confidence set, restatement of Definition~\ref{def:confidence set main}]\label{def:confidence set}
\begin{align*}
\mathcal{M}^k =\bigg\{ (H, \Sc, \Ac, P', R', s_1) :    \,     \forall(h,s,a),   \, \envert{\R'-\R + \langle P'_{h,s,a}-P_{h,s,a},V^*_{h+1}\rangle}\,\,\, 
\leq \sqrt{e_k(h,s,a)}  \bigg\},
\end{align*}
where we set 
\begin{equation}\label{eq: confidence interval}
\sqrt{e_k(h,s,a)} = 
  H\sqrt{ \frac{ \log\left( 2HSA k  \right) }{n^k(h,s,a)+1}}. 
\end{equation}
\end{definition}


\subsection{Martingale Difference Sequences}\label{sec: MDS}

In this section, we give the filtration sets that consists of the history of the algorithm. Later we enumerate the martingale difference sequences needed for the analysis based on these filtration sets. We use the following to denote the history trajectory:
\begin{align*}
\History &\coloneqq  \{ (s^j_l,a^j_l,r^j_l):\text{if } j<k \text{ then}\,
l\in [H],\text{ else if }\,j=k\,\text{ then }\,l \in [h]\} \nonumber\\
\HistoryN &\coloneqq  \mathcal{H}^k_{h}\,\bigcup\, \left\{  w^k(l,s,a):l\in[H],s\in\mathcal{S},a\in\mathcal{A}\right\}.
\end{align*}

With $a^k_h=\pi^k_h(s^k_h)$ as the action taken by \algo{} following the policy $\pi^k_h$ and conditioned on the history of the algorithm, the randomness exists only on the next-step transitions. Specifically, with filtration sets $\{\HistoryN\}_{h,k}$, we define the following notations that is related to the martingale difference sequences (MDS) appeared in the final regret bound:
\begin{align*}
\MDSbk &=\ind\{\mathcal{G}_k\}\left[\mathbb{E}\left[ \,\overline{\delta}^{\polbarM^k}_{h+1,k}(s')\right]  - \overline{\delta}^{\polbarM^k}_{h+1,k}(s^k_{h+1})\right],\\
\MDSck &=\ind\{\mathcal{G}_k\}\left[\mathbb{E}\left[\underline{\delta}^{\polbarM^k}_{h+1,k}(s') \right] - \underline{\delta}^{\polbarM^k}_{h+1,k}(s^k_{h+1})\right],
\end{align*}
where the expectation is over next state $s'$ due to the transition distribution: $\Phk$. 

We use with the filtration sets $\{\mathcal{H}^{k-1}_H\}_{k}$ for the following martingale difference sequence
\begin{align*}
\mathcal{M}_{1,k}^w&= \ind\{\mathcal{G}_k\}\left[\mathbb{E}_{\tilde w|\tilde{\mathcal{G}}_k}\left[ \tilV_{1,k}(s^k_1)\right] -\barV_{1,k}(s^k_1)\right].
\end{align*}

The detailed proof related to martingale difference sequences is presented in Lemma~\ref{lem: MDS concentration}.

\subsection{Events}
For reference, we list the useful events in Table \ref{tab: notation}. 




\section{Proof of Optimism}\label{sec: optimism}
Optimism is required since it is used for bounding the pessimism term in the regret bound calculation. We only care about the probability of a timestep 1 in an episode $k$ being optimistic. The following proof is adapted from \citet{zanette2020frequentist, russo2019worst}. 
\begin{lemma}[Optimism with a constant probability, restatement of Lemma \ref{lem: Optimism Main}]\label{lem: Optimism}
Conditioned on history $\mathcal H^{k-1}_H$, we have $$\mathbb{P}\left(\barV^{\polbarMk}_{1,k}(s^k_1) \geq V^*_{1}(s^k_1)\,|\mathcal{C}_k\right) \geq \Phi(-\sqrt 2),$$ where $\mathcal{C}_k$ refers to the event that $\hat M^k\in \mathcal{M}^k$, where $\mathcal M^k$ is the confidence set defined in Eq (\ref{eq: confidence interval}).

In addition, when $0<\delta < 4\Phi(-\sqrt 2)$, we have 
$$\mathbb{P}\left(\barV^{\polbarMk}_{1,k}(s^k_1) \geq V^*_{1}(s^k_1)\,|\mathcal{G}_k\right) \geq \Phi(-\sqrt 2)/2.$$
\end{lemma}
\begin{proof}
The analysis is valid for any episode $k$ and hence we skip $k$ from the subsequent notations in this proof. 

For the first result, we conditioned all the discussions on event $\mathcal C_k$. Let $(s_1,\cdots,s_H)$ be the random sequence of states drawn by the policy $\pi^*$ (optimal policy under true MDP $M$) in the estimated MDP $\barM$ and $a_h=\pi^*_h(s_h)$. By the property of Bayesian linear regression and Bellman equation, for any $s_h$ (or more specifically $(h,s_h,a_h)$) that is not clipped, we have
\begin{align*}
	&~\barV_h(s_h)-V^{*}_h(s_h)\\
	\ge &~\barQ_h(s_h,\pi^*_h(s_h))-Q^{*}_h(s_h,\pi^*_h(s_h))\\
	=&~\hat Q_h(s_h,\pi^*_h(s_h))-Q^{*}_h(s_h,\pi^*_h(s_h))\\
	 =&~ \RHath + \noiseh + \langle  \PHath \, , \, \barV_{h+1}  \rangle   -\Rh   -\langle  \Ph \, , \, V^{*}_{h+1}  \rangle \\
	=&~ \RHath -\Rh +\langle  \PHath \, , \,\barV_{h+1} - V^{*}_{h+1}  \rangle+\langle  \PHath-\Ph \, , \, V^{*}_{h+1}  \rangle + \noiseh.\\
	\overset{a}{\ge}&~ \langle  \PHath \, , \,\barV_{h+1} - V^{*}_{h+1}  \rangle + \noiseh - \sqrt{e(h,s_h, a_h)},
\end{align*}
where step $(a)$ is from the definition of the confidence sets (Definition~\ref{def:confidence set}).

For any $s_h$ that is clipped, we have
\begin{align*}
\barV_h(s_h)-V^{*}_h(s_h)=H-h+1-V^{*}_h(s_h)\ge 0.
\end{align*}

From the above one-step expansion, we know that we will keep accumulating $\noiseh - \sqrt{e(h,s_h, a_h)}$ when unrolling an trajectory until clipping happens. Define $d(h,s)$ as the probability of the random sequence $(s_1,\ldots,s_H)$ that satisfies $s_h=s$ and no clipping happens at $s_1,\ldots,s_{h}$. Unrolling from timestep 1 to timestep $H$ and noticing $a_h=\pi^*_h(s_h)$ gives us
\begin{align*}
&~\frac{1}{H} \del{\barV_{1}(s_1) -V^{*}_1(s_1)}\nonumber \\
\geq&~ \frac{1}{H} \sum_{s\in\Scal,1\le h \le H}d(h,s)\sbr{w_{h,s,\pi^*_h(s)} - \sqrt{e(h,s, \pi^*_h(s))}}\nonumber\\
\geq&~  \left( \sum_{s\in\Scal,1\le h \le H}(d(h,s)/H)w_{h,s,\pi^*_h(s)}\right)  - \sqrt{HS} \sqrt{ \sum_{s\in \mathcal{S}, 1\leq h\leq H} (d(h,s)/H)^2  e(h,s, \pi^*_h(s))}\label{eq: optimism lem 1} \\
:=&~ X(w).\nonumber
\end{align*}
The first inequality is due to the definition of $d(h,s)$, and the second inequality is due to Cauchy-Schwartz. 
Since 
\[(d(h,s)/H)w_{h,s,\pi^*_h(s)}\sim\mathcal{N}\del{0,(d(h,s)/H)^2HS e(h,s, \pi^*_h(s))/2},
\]
we get
\begin{equation*}
	X(w) \sim \mathcal{N}\left(  - \sqrt{HS} \sqrt{ \sum_{s\in \mathcal{S}, 1\leq h\leq H} (d(h,s)/H)^2  e(h,s, \pi^*_h(s))},\,  HS \sum_{s\in \mathcal{S}, 1\leq h\leq H} (d(h,s)/H)^2  e(h,s, \pi^*_h(s))/2 \right).
\end{equation*}
Upon converting to standard Gaussian distribution it follows that
\begin{align*}
    \mathbb{P}\del{X(W) \geq 0}=\mathbb{P}\left(\mathcal{N}(0,1)\ge\sqrt{2}\right)= \Phi(-\sqrt{2}).
\end{align*}
Therefore $\mathbb{P}\del{\barV_{1}(s_1) \ge V^{*}_1(s_1)\mid \mathcal C_k } \geq \Phi(-\sqrt 2)$.

For the second part, Lemma~\ref{lem: intersection event} tells us that $P(\mathcal G_k|\mathcal C_k)\ge 1-\delta/8$. Applying the law of total iteration yields
\begin{align*}
\mathbb{P}\del{\barV_{1}(s_1) \ge V^{*}_1(s_1)\mid \mathcal C_k } =&~\mathbb{P}\del{\mathcal G_k|\mathcal C_k}\mathbb{P}\del{\barV_{1}(s_1) \ge V^{*}_1(s_1)\mid \mathcal G_k,\mathcal C_k }+\mathbb{P}\del{\mathcal G_k^\complement|\mathcal C_k}\mathbb{P}\del{\barV_{1}(s_1) \ge V^{*}_1(s_1)\mid \mathcal G_k^\complement,\mathcal C_k}\\
\le&~\mathbb{P}\del{\barV_{1}(s_1) \ge V^{*}_1(s_1)\mid \mathcal G_k } + \delta/8.
\end{align*}
Therefore, we get
\begin{align*}
\mathbb{P}\del{\barV_{1}(s_1) \ge V^{*}_1(s_1)\mid \mathcal G_k } \ge \Phi(-\sqrt 2)- \delta/8\ge \Phi(-\sqrt 2) /2.
\end{align*}
This completes the proof.
\end{proof}


\section{Concentration of Events}\label{sec: concentration of events}
\begin{lemma}[Bound on the confident set, restatement of Lemma \ref{lem: confidence interval lemma main}]\label{lem: confidence interval lemma}
    $\sum_{k=1}^{\infty} \mathbb{P}\left(\Econf \right)=\sum_{k=1}^{\infty} \mathbb{P}(\hat{M}^k \notin \mathcal{M}^k ) \leq 2006HSA$.
\end{lemma}
\begin{proof}
	Similar as \cite{russo2019worst}, we construct ``stack of rewards'' as in \cite{lattimore2020bandit}. For every tuple $z=(h,s,a)$, we generate two i.i.d sequences of random variables $r_{z,n}\sim \R$ and $s_{z,n}\sim P_{h,s,a}(\cdot)$. Here  $r_{(h,s,a),n}$ and $s_{(h,s,a),n}$ denote the reward and state transition generated from the $n$th time action $a$ is played in state $s$, timestep $h$. Set
	\[ 
	Y_{z,n} = r_{z,n} + V_{h+1}^*(s_{z,n}) \qquad n\in \mathbb{N}.
	\]
	They are i.i.d, with $Y_{z,n} \in [0,H]$ since $\| V_{h+1}^*\|_{\infty } \leq H-1$, and satisfies 
	\[ 
	\E[Y_{z,n}] = \R + \langle P_{h,s,a} \, , \, V^*_{h+1} \rangle.
	\]  
	
	Now let $n=n^k(h,s,a)$. First consider the case $n\ge 0$. From the definition of empirical MDP, we have 
	\[
	\RHatk + \langle \PHatk  \, ,\, V^*_{h+1} \rangle = \frac{1}{n+1} \sum_{i=1}^{n} Y_{(h,s,a), i}=\frac{1}{n} \sum_{i=1}^{n} Y_{(h,s,a), i}-\frac{1}{n(n+1)} \sum_{i=1}^{n} Y_{(h,s,a), i}.
	\]
    Applying triangle inequality gives us 
    \begin{align*}
    &~\Prob\left(  \left| \RHatk -\Rk+ \langle \PHatk - \Pk \, ,\, V^*_{h+1} \rangle    \right|   \geq  H\sqrt{\frac{\log(2/\delta_n)}{2n}} \right)\\ =&~\Prob\left(  \left|\frac{1}{n} \sum_{i=1}^{n} Y_{(h,s,a), i}-\Rk- \langle   \Pk \, ,\, V^*_{h+1} \rangle -\frac{1}{n(n+1)} \sum_{i=1}^{n} Y_{(h,s,a), i}   \right|   \geq  H\sqrt{\frac{\log(2/\delta_n)}{2n}} \right) \\ 
    \leq &~\Prob\left(  \left|\frac{1}{n} \sum_{i=1}^{n} Y_{(h,s,a), i}-\Rk- \langle   \Pk \, ,\, V^*_{h+1} \rangle  \right|   + \left|\frac{1}{n(n+1)} \sum_{i=1}^{n} Y_{(h,s,a), i}   \right|   \geq  H\sqrt{\frac{\log(2/\delta_n)}{2n}} \right)\\
    \leq &~\Prob\left(  \left|\frac{1}{n} \sum_{i=1}^{n} Y_{(h,s,a), i}-\Rk- \langle   \Pk \, ,\, V^*_{h+1} \rangle \right|   + \frac{1}{n+1} H  \geq  H\sqrt{\frac{\log(2/\delta_n)}{2n}} \right)\\
    = &~\Prob\left(  \left|\frac{1}{n} \sum_{i=1}^{n} Y_{(h,s,a), i}-\Rk- \langle   \Pk \, ,\, V^*_{h+1} \rangle  \right|   \geq  H\sqrt{\frac{\log(2/\delta_n)}{2n}}-\frac{1}{n+1} H\right).
    \end{align*}
    
    When $n\ge 126$, we have
    \begin{align*}
    &~\Prob\left(  \left|\frac{1}{n} \sum_{i=1}^{n} Y_{(h,s,a), i}-\Rk- \langle   \Pk \, ,\, V^*_{h+1} \rangle \right|   \geq  H\sqrt{\frac{\log(2/\delta_n)}{2n}}-\frac{1}{n+1} H\right)\\
    \le &~\Prob\left(  \left|\frac{1}{n} \sum_{i=1}^{n} Y_{(h,s,a), i}-\Rk- \langle   \Pk \, ,\, V^*_{h+1} \rangle  \right|   \geq  H\sqrt{\frac{\log(2/\delta_n)}{2n}}-\frac{H}{8}\sqrt{\frac{\log(2/\delta_n)}{2n}} \right)\\
    = &~\Prob\left(  \left|\frac{1}{n} \sum_{i=1}^{n} Y_{(h,s,a), i}-\Rk- \langle   \Pk \, ,\, V^*_{h+1} \rangle  \right|   \geq  \frac{7H}{8}\sqrt{\frac{\log(2/\delta_n)}{2n}}\right). 
    \end{align*}
    
    By Hoeffding's inequality, for any $\delta_n \in (0,1)$,
	\[ 
	\Prob\left(  \left| \frac{1}{n}\sum_{i=1}^{n} Y_{(h,s,a),i}  - \Rk - \langle \Pk \, , \, V^*_{h+1} \rangle \right|  \geq  \frac{7H}{8}\sqrt{\frac{\log(2/\delta_n)}{2n}} \right)  \leq \sqrt[64]{2^{15}\delta_n^{49}}.
	\]
	
	For $\delta_n=\frac{1}{HSAn^2}$, a union bound over $HSA$ values of $z=(h,s,a)$ and all possible $n\ge 127$ yields
	\begin{align*}
	&~\Prob\left( \bigcup_{h\in[H],s\in[S],a\in[A], n\ge126} \left\{ \left| \frac{1}{n}\sum_{i=1}^{n} Y_{(h,s,a),i}  - \Rk - \langle \Pk \, , \, V^*_{h+1} \rangle \right|  \geq  H\sqrt{\frac{\log(2/\delta_n)}{2n}}\right\} \right)\\
	\le&~\Prob\left( \bigcup_{h\in[H],s\in[S],a\in[A], n\ge126} \left\{ \left| \frac{1}{n}\sum_{i=1}^{n} Y_{(h,s,a),i}  - \Rk - \langle \Pk \, , \, V^*_{h+1} \rangle \right|  \geq  \frac{7H}{8}\sqrt{\frac{\log(2/\delta_n)}{2n}}\right\} \right)\\
	\leq&~ \sum_{s=1}^S\sum_{a=1}^A\sum_{h=1}^H\sum_{n=126}^{\infty} \sqrt[64]{2^{15}\left(\frac{1}{HSAn^2}\right)^{49}}\\
	=&~ (HSA) \sum_{n=126}^{\infty} \sqrt[64]{2^{15}\left(\frac{1}{HSAn^2}\right)^{49}}\\
	\le &~ 2(HSA)^{15/64} \sum_{n=1}^{\infty} \left(\frac{1}{n}\right)^{49/32}\\
	\le &~ 2(HSA)^{15/64} \left(\int_{x=1}^{\infty} \left(\frac{1}{x}\right)^{49/32}dx+1\right)\\
	\le &~ 6(HSA)^{15/64}.
	\end{align*}

    For $1\le n \le 125$, we instead have 
    \begin{align*}
    &~\Prob\left(  \left|\frac{1}{n} \sum_{i=1}^{n} Y_{(h,s,a), i}-\Rk- \langle   \Pk \, ,\, V^*_{h+1} \rangle \right|   \geq  H\sqrt{\frac{\log(2/\delta_n)}{2n}}-\frac{1}{n+1} H \right)\\
    \le &~\Prob\left( \left|\frac{1}{n} \sum_{i=1}^{n} Y_{(h,s,a), i}-\Rk- \langle   \Pk \, ,\, V^*_{h+1} \rangle  \right|   \geq  H\sqrt{\frac{\log(2/\delta_n)}{2n}}-\frac{H}{2}\sqrt{\frac{\log(2/\delta_n)}{2n}} \right)\\
    = &~\Prob\left(  \left|\frac{1}{n} \sum_{i=1}^{n} Y_{(h,s,a), i}-\Rk- \langle   \Pk \, ,\, V^*_{h+1} \rangle \right|   \geq  \frac{H}{2}\sqrt{\frac{\log(2/\delta_n)}{2n}} \right).    
    \end{align*}
    
    By Hoeffding's inequality, for any $\delta_n \in (0,1)$, we have
    \[ 
	\Prob\left(  \left| \frac{1}{n}\sum_{i=1}^{n} Y_{(h,s,a),i}  - \Rk - \langle \Pk \, , \, V^*_{h+1} \rangle \right|  \geq  \frac{H}{2}\sqrt{\frac{\log(2/\delta_n)}{2n}} \right)  \leq \sqrt[4]{8\delta_n}.
	\]
	
	For $\delta_n=\frac{1}{HSAn^2}$, a union bound over $HSA$ values of $z=(h,s,a)$ and all possible $1\le n \le 125$ gives 
	\begin{align*}
	&~\Prob\left( \bigcup_{h\in[H],s\in[S],a\in[A],1\le n\le125} \left\{ \left| \frac{1}{n}\sum_{i=1}^{n} Y_{(h,s,a),i}  - \Rk - \langle \Pk \, , \, V^*_{h+1} \rangle \right|  \geq  H\sqrt{\frac{\log(2/\delta_n)}{2n}}\right\} \right)\\
	\le &~\Prob\left( \bigcup_{h\in[H],s\in[S],a\in[A],1\le n\le125} \left\{ \left| \frac{1}{n}\sum_{i=1}^{n} Y_{(h,s,a),i}  - \Rk - \langle \Pk \, , \, V^*_{h+1} \rangle \right|  \geq  \frac{H}{2}\sqrt{\frac{\log(2/\delta_n)}{2n}}\right\} \right)\\
	\leq&~ \sum_{s=1}^S\sum_{a=1}^A\sum_{h=1}^H\sum_{n=1}^{125} \sqrt[4]{8\frac{1}{HSAn^2}}\\
	=&~ (HSA)^{3/4} \sum_{n=1}^{125} \sqrt[4]{8\frac{1}{n^2}}\\
	\le &~ 2000(HSA)^{3/4} .
	\end{align*}
	
	Combining the above two cases, we have
	\begin{align*}
   &~ \Prob\left( \exists (k,h,s,a) : n>0 , \left| \RHatk -\Rk+ \langle \PHatk - \Pk,V^*_{h+1} \rangle    \right| \geq  H\sqrt{ \frac{ \log\left( 2HSA n  \right) }{2 n}}  \right) \\
   \le &~3(HSA)^{15/64} + 2000(HSA)^{3/4} \\
   \le &~ 2006 HSA.
	\end{align*}
	
	Note that by definition, when $n=n^k(h,s,a)>0$ we have \[
	\sqrt{e_{k}(h,s,a)} \geq  H\sqrt{ \frac{ \log\left( 2HSA n^k(h,s,a)  \right) }{2n^k(h,s,a)}}
	\]
	and hence this concentration inequality holds with $\sqrt{e_k(h,s,a)}$ on the right hand side.

	When $n=n^k(h,s,a)=0$, we have $\RHatk=0$ and $\PHatk(\cdot)=0$ by definition, so we trivially have
	\[ 
	\left| \RHatk -\Rk+ \langle \PHatk - \Pk \, ,\, V^*_{h+1} \rangle    \right|  = | \Rk +\langle \Pk \, ,\, V^*_{h+1} \rangle   | \leq H \leq e_{k}(h,s,a).
	\]
\end{proof}

\begin{lemma}[Bound on the noise]\label{lemma: noise bounds}
For $w^k(h,s,a)\sim \mathcal{N}(0,\sigma^2_k(h,s,a)),$ where $\sigma^2_k(h,s,a) = \frac{H^3S\log(2HSAk)}{2(n^k(h,s,a)+1)}$, we have that for any $k\in[K]$, the event $\mathcal{E}^w_{k}$ holds with probability at least $1-\delta/8$.
\end{lemma}
\begin{proof}
For any fix $s,a,h,k$, the random variable $w^k(h,s,a)$ follows Gaussian distribution $\mathcal{N}(0,\sigma^2_k)$. Therefore, Chernoff concentration bounds~(see e.g. \cite{wainwright_2019}) suggests
\begin{equation*}
    \mathbb{P}\left[|w^k(h,s,a)|\geq t\right]\leq 2\exp\left(-\frac{t^2}{2\sigma^2_k}\right).
\end{equation*}
Substituting the value of $\sigma^2_k$ and rearranging, with probability at least $1-\delta'$, we can write
\begin{equation*}
    \envert{w^k(h,s,a)} \leq \sqrt{\frac{H^3S\log(2HSAk)\log(2/\delta')}{n^k(h,s,a)+1}}.
\end{equation*}
Union bounding over all $s,a,k,h$ (i.e. over state, action, timestep, and episode) imply that $\forall s,a,k,h$, the following hold with probability at least $1-\delta'$,
\begin{equation*}
    \envert{w^k(h,s,a)} \leq \sqrt{\frac{H^3S\log(2HSAk)\log(2SAT/\delta')}{n^k(h,s,a)+1}}.
\end{equation*}

Setting $\delta' = \delta/8$, for any $s\in[S],a\in[A],h\in[H],k\in[K]$, we have that
\begin{equation*}
    \envert{w^k(h,s,a)} \leq \sqrt{\frac{H^3S\log(2HSAk)\log(16SAT/\delta)}{n^k(h,s,a)+1}} \le \gamma_k(h,s,a).
\end{equation*}
Finally recalling the definition of $\mathcal{E}^w_{k}$, we complete the proof. 
\end{proof}

\begin{lemma}[Bounds on the estimated action-value function, restatement of Lemma \ref{lem: est Q function bounded main}]\label{lem: est Q function bounded}
When the events $\Econf$ and $\mathcal{E}^w_k$ hold then for all $(h,s,a)$
\begin{equation*}
    \left| \left(\barQ^{\polbarMk}_{h,k} - Q^*_{h}\right)(s,a) \right| \leq H-h +1.
\end{equation*}
\end{lemma}
\begin{proof}
For simplicity, we set $\barQ^{\polbarMk}_{H+1,k}(s,a) = Q^*_{H+1}(s,a)=0$ and it is a purely virtual value for the purpose of the proof. The proof goes through by backward induction for $h=H+1,H,\ldots,1$. 

Firstly, consider the base case $h=H+1$. The condition $|\barQ^{\polbarMk}_{H+1,k}(s,a) - Q^*_{H+1}(s,a)| =0 \leq H - (H+1) +1 $ directly holds from the definition. 

Now we do backward induction. Assume the following inductive hypothesis to be true 
\begin{equation}\label{Eq: Good Events Q IH}
    \envert{ \left(\barQ^{\polbarMk}_{h+1,k} - Q^*_{h+1}\right)(s,a) } \leq H-h.
\end{equation}
We consider two cases:\\
\textbf{Case 1:}  $n^k(h,s,a)\leq \alpha_k$\\
The Q-function is clipped and hence $\barQ^{\polbarMk}_{h,k} = H-h+1$. By the definition of the optimal Q-function, we have $0\le Q^*_{h}\leq H-h+1$. Therefore it is trivially satisfied that
\begin{equation*}
    \left| \left(\barQ^{\polbarMk}_{h,k} - Q^*_{h}\right)(s,a) \right| \leq H-h +1.
\end{equation*}
\textbf{Case 2:} $n^k(h,s,a) > \alpha_k$\\
In this case, we don't have clipping, so $\overline Q_{h,k}(s,a)=\hat Q_{h,k}(s,a)$. From the property of Bayesian linear regression and Bellman equation, we have the following decomposition
\begin{align*}
	&~\envert{\barQ^{\polbarMk}_{h,k}(s,a)-Q^*_{h}(s,a)} \\
	 =&~ \envert{\hat{R}^k_{h,s,a} + w^k_{h, s, a} + \langle  \hat{P}^k_{h,s, a} \, , \, \barV^{\polbarMk}_{h+1,k}  \rangle   -R^{k}_{h, s, a}   -\langle  P^{k}_{h,s,a} \, , \, V^*_{h+1}  \rangle } \\
	 \leq&~ \underbrace{\envert{\langle \hat{P}^k_{h,s, a}  \, , \, \barV^{\polbarMk}_{h+1,k} - V^*_{h+1}\rangle}}_{(1)} + \underbrace{\envert{\hat{R}^k_{h, s,a} -R^{k}_{h, s, a} +\langle  \hat{P}^k_{h,s, a}- P^{k}_{h,s,a} \, , \, V^*_{h+1}  \rangle}}_{(2)}   +  \underbrace{\envert{w^k_{h, s, a}}}_{(3)}.
\end{align*}
Term (1) is bounded by $H-h$ due to the inductive hypothesis in Eq~(\ref{Eq: Good Events Q IH}). Under the event $\mathcal{C}_k$, term (2) is bounded by $ \sqrt{e_k(h,s,a)} = H\sqrt{ \frac{ \log\left( 2HSA k  \right) }{n^k(h,s,a)+1}}$. Finally, term (3) is bounded by $\gamma_k(h,s,a)$ as the event $\mathcal{E}^w_{h,k}$ holds. With the choice of $\alpha_k$, it follows that the sum of terms (2) and (3) is bounded by 1 as

\begin{equation*}
     \frac{\sqrt{ H^2\log\left( 2HSA k  \right) } +\sqrt{H^3S\log(2HSAk)L}}{\sqrt{n^k(h,s,a)}} <1.
\end{equation*}
Thus the sums of all the three terms is upper bounded by $H-h+1$.
This completes the proof.
\end{proof}

\begin{lemma}[Intersection event probability]\label{lem: intersection event}
For any episode $k\in[K]$, when the event $\Econf$ holds (i.e. $\hat{M}^k\,\in\,\mathcal{M}^k$), the intersection event
intersection event $\overline{\mathcal{E}}_k = \mathcal{E}^{w}_{k} \cap \mathcal{E}^{\overline{Q}^{\polbarMk}}_{k}$ holds with probability at least $1-\delta/8$. In other words, whenever the unperturbed estimated MDP lies in the confidence set (Definition~\ref{def:confidence set}), the each pseudo-noise and the estimated $\barQ$ function are bounded with high probability $1-\delta/8$. Similarly defined, $\tilde{\mathcal{E}}_k$ also holds with probability $1-\delta/8$ when $\Econf$ happens. 
\end{lemma}

\begin{proof}
The event, $\mathcal{E}^w_k$ holds with probability at least $1-\delta/8$ from Lemma~\ref{lemma: noise bounds}. Lemma~\ref{lem: est Q function bounded} gives that whenever $\left(\Econf \cap \mathcal{E}^w_k\right)$ holds then almost surely $\mathcal{E}^{\overline{Q}^{\polbarMk}}_{k}$ holds. Therefore, $\mathcal{E}_k$ holds with probability $1-\delta/8$, whenever $\Econf$ holds.
\end{proof}

\section{Regret Decomposition}
\label{sec: apx_d}
In this section we give a full proof of our main result Theorem~\ref{thm: Regret main result} which is a formal version of Theorem~\ref{thm: high probability regret main}. We will give a high-level sketch proof before jumping into the details of individual parts in Sections~\ref{sec: Estimation bounds} and~\ref{sec: Bounds on Pessimism}.

\begin{thm}\label{thm: Regret main result}
For $0<\delta < 4\Phi(-\sqrt 2)$, \algo\, enjoys the following high probability regret upper bound, with probability at least $1-\delta$, 
\begin{equation*}
    {\rm Reg}(K) = \tilde{\mathrm{O}}\left( H^2S\sqrt{AT}+H^5S^2A\right).
\end{equation*}
\end{thm}

We first decompose the regret expression into several terms and show bounds for each of the individual terms separately. With probability at least $1-\delta/4$, we have
\begin{align}
{\rm Reg}(K) ~=& ~ \sum_{k=1}^{K} \ind\{\Econf\}\left(V_{1}^{*}(s_1^k) - V_{1}^{\pi^k}(s_1^k)\right) +  \sum_{k=1}^{K} \underbrace{\ind\{ \{\Econf\}^{\complement} \}\left(V_{1}^{*}(s_1^k) - V_{1}^{\pi^k}(s_1^k)\right)}_{(1)}\nonumber  \\
\overset{a}{=}& ~ \sum_{k=1}^{K} \ind\{\mathcal{G}_k\}\left(V_{1}^{*}(s_1^k) - V_{1}^{\pi^k}(s_1^k)\right) +  \sum_{k=1}^{K} \underbrace{\ind\{ \{\Econf\}^{\complement} \}\left(V_{1}^{*}(s_1^k) - V_{1}^{\pi^k}(s_1^k)\right)}_{(1)}\nonumber  \\
\leq&~ 
\sum_{k=1}^{K} \ind\{\mathcal{G}_k\}\left( \underbrace{V_{1}^{*}(s_1^k) - \barV^{\polbarMk}_{1,k}(s_1^k)}_{(2)} + \underbrace{\barV^{\polbarMk}_{1,k}(s_1^k) - V_{1}^{\pi^k}(s_1^k)}_{(3)}\right) + 2006H^2SA.\label{eq: regret decom main}
\end{align}
Step (a) holds with probability at least $1-\delta/4$ due to Lemma~\ref{lem: intersection event}. 

Term (1) is upper bounded due to Lemma~\ref{lem: confidence interval lemma} and the fact that $V_{h}^{*}(s^k_h) - V_{h}^{\pi^k}(s^k_h) \leq H,\,\forall\,k\,\in\,[K]$. Term (2), additive inverse of {\sl optimism}, is called {\sl pessimism} \citep{zanette2020frequentist} and is further decomposed in Lemma~\ref{lem: pessimism decomp} and Lemma~\ref{lem:est-underV}. Term (3) is a measure of how well the estimated MDP tracks the true MDP and is called {\sl estimation error}. It is discussed further by Lemma~\ref{lem: Decomposition Lemma}, Lemma~\ref{lem: Decomposition supporting Lemma} and finally decomposed in Lemma~\ref{lem: estimation decomp}. We start with the results that decompose the terms in Eq~(\ref{eq: regret decom main}) and later aggregate them back to complete the proof of Theorem~\ref{thm: Regret main result}.


\subsection{Bound on the Estimation Term}\label{sec: Estimation bounds}
Lemma~\ref{lem: Decomposition Lemma} decomposes the deviation term between the Q-value function and its estimate, and the proof relies on Lemma~\ref{lem: Decomposition supporting Lemma}. This result is extensively used in our analysis. For the purpose of the results in this subsection, we assume the episode index $k$ is fixed and hence dropped from the notation in both the lemma statements and their proofs when it is clear.

\begin{lemma}\label{lem: Decomposition Lemma}
With probability at least $1-\delta/4$, for any $h,k,s_h,a_h$, it follows that
\begin{align*}
&~ \ind\{\mathcal{G}_k\}\left[\barQ_{h}(s_h,a_h) - Q^{\pi}_{h}(s_h,a_h)\right] \\
\leq &~\ind\{\mathcal{G}_k\}\ind\{\Ethreshk\}\left(\PDiffh +\RDiffh + \noiseh + \deltaEPi{h+1}+ \MDSb 
+ 4H\sqrt{\frac{SL}{n(h,s_h,a_h) + 1}}\right)+H\ind\{\mathcal{E}^{\text{th}\;\complement}_k\}.
\end{align*}
\end{lemma}

\begin{proof}
Here the action at any period $h$ is due to the policy of the algorithm, therefore, $a_h=\pi(s_h)$. By the property of Bayesian linear regression and the Bellman equation, we have the following 
\begin{align*}
&~ \ind\{\mathcal{G}_k\}\left[\barQ_{h}(s_h,a_h) - Q^{\pi}_{h}(s_h,a_h)\right] \\
=&~ \ind\{\mathcal{G}_k\}\left[\barQ_{h}(s_h,a_h) - Q^{\pi}_{h}(s_h,a_h)\right]\ind\{\Ethreshk\}+ \ind\{\mathcal{G}_k\}\left[\barQ_{h}(s_h,a_h) - Q^{\pi}_{h}(s_h,a_h)\right]\ind\{\mathcal{E}^{\text{th}\;\complement}_{h,k}\}\\
\le&~ \ind\{\mathcal{G}_k\}\ind\{\Ethreshk\}\left[\hat Q_{h}(s_h,a_h) - Q^{\pi}_{h}(s_h,a_h)\right]+H \ind\{\mathcal{E}^{\text{th}\;\complement}_k\}\\
=&~\ind\{\mathcal{G}_k\}\ind\{\Ethreshk\}\left[\langle\PHath ,\barV_{h+1}\rangle - \langle \Ph, V^\pi_{h+1}\rangle+ \RHath-\Rh+\noiseh\right]+H \ind\{\mathcal{E}^{\text{th}\;\complement}_k\}\\
=&~\ind\{\mathcal{G}_k\}\ind\{\Ethreshk\}\left[\langle\PHath ,\barV_{h+1}\rangle - \langle \Ph, V^\pi_{h+1}\rangle+ \RHath-\Rh+\noiseh\right.\\
&~+\left.\langle\PHath- \Ph,V^*_{h+1}\rangle-\langle\PHath- \Ph,V^*_{h+1}\rangle\right]+H \ind\{\mathcal{E}^{\text{th}\;\complement}_k\}\\
=&~ \ind\{\mathcal{G}_k\}\ind\{\Ethreshk\}\left[\PDiffh+\RDiffh+\noiseh+ \langle \Ph, \barV_{h+1}-V^\pi_{h+1}\rangle\right]\\
&~+ \ind\{\mathcal{G}_k\}\ind\{\Ethreshk\}\langle\PHath- \Ph,\barV_{h+1}-V^*_{h+1}\rangle+H \ind\{\mathcal{E}^{\text{th}\;\complement}_k\}\\
\overset{a}{\leq}& ~\ind\{\mathcal{G}_k\}\ind\{\Ethreshk\}\left[\PDiffh +\RDiffh + \noiseh + \overline{\delta}^{\pi}_{h+1}(s_{h+1})+ \MDSb
+ 4H\sqrt{\frac{SL}{n(h,s_h,a_h) + 1}} 
\right]+H \ind\{\mathcal{E}^{\text{th}\;\complement}_k\},
\end{align*}
where step $(a)$ follows from Lemma \ref{lem: Decomposition supporting Lemma} and by adding and subtracting $\barV^{\pi}_{h+1}(s_{h+1})-V^\pi_{h+1}(s_{h+1})$ to create $\MDSb$.
\end{proof}

\begin{lemma}\label{lem: Decomposition supporting Lemma}
With probability at least $1-\delta/4$, for any $h,k,s,a$, it follows that
\begin{align*}
&~\ind\{\mathcal{G}_k\}\langle\PHat- \Pk,V^*_{h+1}-\barV_{h+1}\rangle\le 4H\sqrt{\frac{SL}{n(h,s,a) + 1}}.
\end{align*}
\end{lemma}
\begin{proof}
Firstly, applying the Cauchy-Schwarz inequality, we get 
\begin{align*}
&~\ind\{\mathcal{G}_k\}\langle\PHat- \Pk,V^*_{h+1}-\barV_{h+1}\rangle \le \|\PHat- \Pk\|_1\|\ind\{\mathcal{G}_k\}(V^*_{h+1}-\barV_{h+1})\|_\infty.
\end{align*}
Since Lemma \ref{lem: est Q function bounded} implies $\|\ind\{\mathcal{G}_k\}(V_{h+1}^*-\barV_{h+1})\|_\infty \le H$, it suffices to bound $\|\PHat- \Pk\|_1$. Note that for any vector $v\in\RR^S$, we have $$\|v\|_1=\sup_{u\in\{-1,+1\}^S}u^\top v.$$

Hence, we will prove the concentration for $u^\top(\PHat- \Pk)$. 

If the visiting time $n^k(h,s,a)=0$, we know that $\|\PHat- \Pk\|_1\le 4H\sqrt{SL}$, which means the final bound holds. Now we consider the case that $n^k(h,s,a)\ge 1$. For any fixed $h,s,a,$ $u\in\{-1,+1\}^S$, and assume given number of visits $n>0$ in $(h,s,a)$ before step $T$, (i.e. $n^k(h,s,a)=n$, but for simplicity we still use $n^k(h,s,a)$ in the analysis below). Applying Hoeffding's inequality to the transition probabilities obtained from $n$ observations, with probability at least $1-\delta'$, we have 
\begin{align*}
    u^\top\left( \frac{1}{n^k(h,s,a)}\sum_{l=1}^{k-1}\ind\{(s_h^l,a_h^l,s_h^{l+1})=(s,a,\cdot)\} -  \Pk(\cdot)\right)\le 2\sqrt{\frac{\log(2/\delta')}{2n^k(h,s,a)}}.
\end{align*}
This is because $u^\top  \left(\frac{1}{n^k(h,s,a)}\sum_{l=1}^{k-1}\ind\{(s_h^l,a_h^l,s_h^{l+1})=(s,a,\cdot)\}\right) $ is the average of i.i.d. random variables $u^\top \mathbf{e}_{s'}$ with bounded range $[-1,1]$. Notice here we have fixed $n^k(h,s,a)=n$, so $n^k(h,s,a)$ is not a random variable.

By triangle inequality, we have
\begin{align*}
&~\left|\PHatk(\cdot)-\frac{1}{n^k(h,s,a)}\sum_{l=1}^{k-1}\ind\{(s^l_h,a^l_h,s_{h+1}^l)=(s,a,\cdot)\} \right|\\
=&~\frac{1}{n^k(h,s,a)(n^k(h,s,a)+1)}\sum_{l=1}^{k-1}\ind\{(s_h^l,a_h^l,s_h^{l+1})=(s,a,\cdot)\}\\
\le&~\frac{1}{n^k(h,s,a)},
\end{align*}
where the last step is by noticing visiting $(s_h^l,a_h^l,s_h^{l+1})=(s,a,\cdot)$ implies visiting $(h,s,a)$.

Therefore, we get 
\begin{align*}
    u^\top\left( \hat P^k_{h,s,a} -  P^k_{h,s,a}\right)\le 3\sqrt{\frac{\log(2/\delta')}{2n^k(h,s,a)}}.
\end{align*}

Finally, union bounding over all $h,s,a$, $u\in\{-1,+1\}^S$, $n^k(h,s,a)\in[K]$ and set $\delta=\delta'/(2^SSAT)$, we get 
\begin{align*}
    u^\top\left( \hat P^k_{h,s,a} -  P^k_{h,s,a}\right)\le 3\sqrt{\frac{SL}{n^k(h,s,a)}}\le 4\sqrt{\frac{SL}{n^k(h,s,a) + 1}}.
\end{align*}

This implies $\|\PHatk- \Pk\|_1\le 4\sqrt{\frac{SL}{n^k(h,s,a)+ 1}}$, which completes the proof.
\end{proof}

The following Lemma \ref{lem: estimation decomp} is the $V$ function version of its $Q$ function version counterpart in Lemma \ref{lem: Decomposition Lemma}. It is applied in the proof of final regret decomposition in Theorem~\ref{thm: Regret main result}.
\begin{lemma}\label{lem: estimation decomp}
With probability at least $1-\delta/4$, for any $h,k,s_h,a_h$, the following decomposition holds
\begin{align*}
& ~\ind\{\mathcal{G}_k\}\left[\barV_{h}(s_h) - V^{\pi}_{h}(s_h)\right] \\
\leq& ~\ind\{\mathcal{G}_k\}\ind\{\Ethreshk\}\left(\PDiffh +\RDiffh + \noiseh +\overline{\delta}^{\pi}_{h+1}(s_{h+1})+ \MDSb
+ 4H\sqrt{\frac{SL}{n(h,s_h,a_h) + 1}}\right)+H \ind\{\mathcal{E}^{\text{th}\;\complement}_k\}.
\end{align*}
\end{lemma}
\begin{proof}
With $a_h$ as the action taken by the algorithm $\pi(s_h)$, it follows that $\barV_{h}(s_h) = \barQ_{h}(s_h,a_h)$ and $V^{\pi}_{h}(s_h) = Q^{\pi}_{h}(s_h,a_h)$. Thus, the proof follows by a direction application of Lemma~\ref{lem: Decomposition Lemma}.
\end{proof}

\subsection{Bound on the Pessimism Term}\label{sec: Bounds on Pessimism}
In this section, we will upper bound the pessimism term with the help of the probability of being optimistic and the bound on the estimation term. The approach generally follows Lemma~G.4 of~\cite{zanette2020frequentist}. The difference here is that we also provide a bound for $V_{1}^{*}(s^k_1) - \underV_{1,k}(s^k_1)$. This difference enable us to get stronger bounds in the tabular setting as compared to~\cite{zanette2020frequentist}. The pessimism term will be decomposed to the two estimation terms $\barV^{\polbarMk}_{1,k}(s^k_1) - V^{\pi^k}_{1}(s^k_1)$ and $V_{1}^{\pi^k}(s^k_1)-\underV^{\polbarMk}_{1,k}(s^k_1)$, and the martingale difference term $\MDSfkind{1}$.

\begin{lemma}[Restatement of Lemma~\ref{lem: pessimism decomp main}]\label{lem: pessimism decomp}
For any $k$, the following decomposition holds,
\begin{align}\label{eq:lem9-1-a}
&~\ind\{\mathcal{G}_k\}\left(V_{1}^{*}(s^k_1) - \barV^{\polbarMk}_{1,k}(s^k_1)\right)\leq \ind\{\mathcal{G}_k\}\left(V_{1}^{*}(s^k_1) - \underV^{\polbarMk}_{1,k}(s^k_1)\right) \nonumber\\
\leq &~C \ind\{\mathcal{G}_k\}\left(\barV^{\polbarMk}_{1,k}(s^k_1) - V^{\pi^k}_{1}(s^k_1)+V_{1}^{\pi^k}(s^k_1)-\underV^{\polbarMk}_{1,k}(s^k_1)+ \MDSfkind{1}\right),
\end{align}
where $\ind\{\mathcal{G}_k\}\left[V_{1}^{\pi^k}(s^k_1)-\underV^{\polbarMk}_{1,k}(s^k_1)\right]$ will be further bounded in Lemma~\ref{lem:est-underV}.
\end{lemma}

\begin{proof}
For the purpose of analysis we use two ``virtual'' quantities $\tilV^{\polbarMk}_{1,k}(s^k_1)$ and $\underV^{\polbarMk}_{1,k}(s^k_1)$, which are formally stated in the Definitions \ref{def: tilde V} and \ref{def: under V} respectively. Thus we can define the event $\tilde{\mathcal{O}}_{1,k} \overset{def}{=} \left\{ \tilV^{\polbarMk}_{1,k}(s^k_1) \geq V^*_{1}(s^k_1) \right\}$. For simplicity of exposition, we skip showing dependence on $k$ in the following when it is clear. 

By Definition~\ref{def: tilde V}, we know that $\barV_1(s_1)$ and $\tilV_1(s_1)$ are identically distributed conditioned on last round history $\mathcal{H}^{k-1}_H$. From Definition~\ref{def: under V}, under event $\mathcal{G}_k$, it also follows that $\underV_1(s_1) \leq \barV_1(s_1)$.

Since $\underV_1(s_1) \leq \barV_1(s_1)$ under event $\mathcal{G}_k$, we get
\begin{align}
    \ind\{\mathcal{G}_k\}\left[V^*_1(s_1) - \barV_1(s_1)\right] \leq \ind\{\mathcal{G}_k\}\left[V^*_1(s_1) - \underV_1(s_1)\right].\label{eq: lemma pes 1}
\end{align}

We also introduce notation $\mathbb{E}_{\tilde w}\sbr{\cdot}$ to denote the expectation over the pseudo-noise $\tilde w$ (recall that $\tilde w$ discussed in Definition \ref{def: tilde V}). Under event $\EoptTwo$, we have $\tilV_{1}(s_1) \geq V^*_{1}(s_1)$. Since $V^*_{1}(s_1)$ does not depend on $\tilde{w}$, we get $V^*_1(s_1) \leq \mathbb{E}_{\tilde{w}|\EoptTwo,\tilde{\mathcal{G}}_k}\sbr{ \tilV_1(s_1)}$.
Using the similar argument for $\underV_1(s_1)$, we know that $\underV_1(s_1)=\mathbb{E}_{\tilde{w}|\EoptTwo,\tilde{\mathcal{G}}_k}\left[\underV_1(s_1)\right]$. Subtracting this equality from the inequality $V^*_1(s_1) \leq \mathbb{E}_{\tilde{w}|\EoptTwo,\tilde{\mathcal{G}}_k}\sbr{ \tilV_1(s_1)}$, it follows that
\begin{align}
    V^*_1(s_1) - \underV_1(s_1) \leq \mathbb{E}_{\tilde{w}|\EoptTwo,\tilde{\mathcal{G}}_k}\left[ \tilV_1(s_1)-\underV_1(s_1)\right]\label{eq: pessimism 3 old}.
\end{align}

Therefore, we have 
\begin{align*}
    \ind\{\mathcal{G}_k\}\left[V^*_1(s_1) - \underV_1(s_1)\right] \leq\ind\{\mathcal{G}_k\} \mathbb{E}_{\tilde{w}|\EoptTwo,\tilde{\mathcal{G}}_k}\left[\tilV_1(s_1)-\underV_1(s_1)\right].
\end{align*}

From the law of total expectation, we can write
\begin{align}
&~\mathbb{E}_{\tilde{w}|\tilde{\mathcal G}_k}\left[ \tilV_1(s_1)-\underV_1(s_1)\right] \nonumber\\
=&~ \mathbb{P}(\EoptTwo|\tilde{\mathcal G}_k)\mathbb{E}_{\tilde{w}|\EoptTwo,\tilde{\mathcal G}_k}\left[ \tilV_1(s_1)-\underV_1(s_1)\right] + \mathbb{P}(\EoptTwo^{\complement}|\tilde{\mathcal G}_k)\mathbb{E}_{\tilde{w}|\EoptTwo^{\complement},\tilde{\mathcal G}_k}\left[ \tilV_1(s_1)-\underV_1(s_1)\right].\label{eq: law of total exp}
\end{align}

Since $\tilV_1(s_1)-\underV_1(s_1) \geq 0$ under event $\tilde{\mathcal{G}}_k$, multiplying both sides of Eq~(\ref{eq: law of total exp}) by $\ind\{\mathcal{G}_k\}$, relaxing the second term on RHS to 0 and rearranging yields
\begin{align}
\ind\{\mathcal{G}_k\}\mathbb{E}_{\tilde{w}|\EoptTwo,\tilde{\mathcal{G}}_k}\left[ \tilV_1(s_1)-\underV_1(s_1)\right] \leq  \frac{1}{\mathbb{P}(\EoptTwo|\tilde{\mathcal{G}}_k)}\ind\{\mathcal{G}_k\}\mathbb{E}_{\tilde{w}|\tilde{\mathcal{G}}_k}\left[\tilV_1(s_1)-\underV_1(s_1)\right].\label{eq: pessimism 4}
\end{align}

Noticing $\tilV$ is an independent sample of $\barV$, we can invoke Lemma \ref{lem: Optimism Main} for $\tilV$, and it follows that $\mathbb{P}(\EoptTwo|\tilde{\mathcal{G}}_k)\geq \OptPr$. Set $C= \frac{1}{\OptPr}$ and consider 
\begin{align}
\ind\{\mathcal{G}_k\}\mathbb{E}_{\tilde{w}|\tilde{\mathcal{G}}_k}\left[\tilV_1(s_1)-\underV_1(s_1)\right] 
     = &~ \underbrace{\ind\{\mathcal{G}_k\}\left(\mathbb{E}_{\tilde{w}|\tilde{\mathcal{G}}_k}\left[ \tilV_1(s_1)\right] -\barV_1(s_1)\right)}_{(1)} + \,\ind\{\mathcal{G}_k\} \underbrace{\left(\barV_1(s_1) -\underV_1(s_1)\right)}_{(2)},
     \label{eq: pessimism 1}
\end{align}
where the equality is due to $\tilde w$ is independent of $\underV_1(s_1)$.

Since $\barV_1(s_1)$ and $\tilV_1(s_1)$ are identically distributed from the definition, we  will later show term (1) $\ind\{\mathcal{G}_k\}\left[\mathbb{E}_{\tilde w|\tilde{\mathcal{G}}_k}\left[\tilV_1(s_1)\right] -\barV_1(s_1)\right]:=\MDSfind{1}$ is a martingale difference sequence in Lemma~\ref{lem: MDS concentration}. Term (2) can be further decomposed as
\begin{align}
    \barV_1(s_1) -\underV_1(s_1) 
    = \underbrace{\barV_1(s_1) - V^{\pi}_1(s_1)}_{(3)} +\, \underbrace{V^{\pi}_1(s_1) -\underV_1(s_1)}_{(4)}.
    \label{eq: pessimism 2}
\end{align}
Term (3) in Eq~(\ref{eq: pessimism 2}) is same as {\sl estimation} term in Lemma~\ref{lem: estimation decomp}. For term (4), to make it clearer, we will show a bound separately in Lemma~(\ref{lem:est-underV}).

Combining Eq~(\ref{eq: pessimism 4}), (\ref{eq: pessimism 1}), and (\ref{eq: pessimism 2}) gives us that 
\begin{align*}
&~\ind\{\mathcal{G}_k\} \mathbb{E}_{\tilde{w}|\EoptTwo,\tilde{\mathcal{G}}_k}\left[\tilV_1(s_1)-\underV_1(s_1)\right]\nonumber\\
\leq&~ C \ind\{\mathcal{G}_k\}\left(V_{1}^{\pi^k}(s^k_1)-\underV^{\polbarMk}_{1,k}(s^k_1)+\barV^{\polbarMk}_{1,k}(s^k_1) - V^{\pi^k}_{1}(s^k_1) + \MDSfkind{1}\right).
\end{align*}
This completes the proof.
\end{proof}

In Lemma~\ref{lem:est-underV}, we provide a missing piece in Lemma~\ref{lem: pessimism decomp}. It will be applied when we do the regret decomposition of major term in Theorem~\ref{thm: Regret main result}.

\begin{lemma}\label{lem:est-underV}
With probability at least $1-\delta/4$, for any $h,k,s_h^k,a_h^k$, the following decomposition holds with the intersection event $\mathcal{G}_k$
\begin{align}\label{eq:lem9-2}
&~\ind\{\mathcal{G}_k\}\left[V_{h}^{\pi^k}(s^k_h)-\underV_{h,k}^{\polbarMk}(s^k_h)\right] \\
\leq&~\ind\{\mathcal{G}_k\}\ind\{\Ethreshk\}\left(-\PDiffhk -\RDiffhk - \noisehUk +\underline{\delta}^{\pi^k}_{h+1,k}(s^k_{h+1})+ \MDSck
+ 4H\sqrt{\frac{SL}{n^k(h,s_h^k,a_h^k) + 1}}\right)+H \ind\{\mathcal{E}^{\text{th}\;\complement}_k\}.\nonumber
\end{align}
\end{lemma}
\begin{proof}
We continue to show how to bound term (4) in Lemma~\ref{lem: pessimism decomp} and we will also drop the superscript $k$ here.

Noticing that $a_h$ as the action chosen by the algorithm $\pi(s_h)$, we have $V^{\pi}_h(s_h) = Q^{\pi}_h(s_h,a_h)$. By the definition of value function $\underV_h(s_h) =\max_{a\,\in\,\mathcal{A}}\underQ_h(s_h,a)$. This gives $\underQ_h(s_h,a_h)\leq \underV_h(s_h)$. Hence,
\begin{equation*}
    V^{\pi}_h(s_h) -\underV_h(s_h) = Q^{\pi}_h(s_h,a_h) -\underV_h(s_h) \leq Q^{\pi}_h(s_h,a_h) -\underQ_h(s_h,a_h).
\end{equation*}

From the definition of $\underline V_h$, we know that its noise satisfies $|\underline w(h,s,a)|\le \gamma(h,s,a)$. Therefore, we can show a version of Lemma~\ref{lem: est Q function bounded} for $\underline V_h$ and get $\|\ind\{\mathcal{G}_k\}(V_{h+1}^*-\underline V_{h+1})\|_\infty \le H$. This implies the version of Lemma~\ref{lem: Decomposition supporting Lemma} for $\underline V_h$ would hold. Since the decomposition and techniques in Lemma~\ref{lem: Decomposition Lemma} only utilize the property that $\barQ_{h}$ is the solution of the Bayesian linear regression and the Bellman equation for $Q^\pi_h$, we can directly get another version for instance $\underQ_{h}$. Also noticing that we flip the sign of $V_h^\pi(s_h)-\underV_h(s_h)$, therefore, we obtain the following decomposition for the term (4) in Lemma~\ref{lem: pessimism decomp}
\begin{align*}
&~\ind\{\mathcal{G}_k\}\left[V_h^{\pi}(s_h)-\underV_h(s_h)\right]\nonumber \\
\leq &~\ind\{\mathcal{G}_k\}\ind\{\Ethreshk\}\left(-\PDiffh -\RDiffh - \underline {w}_{h,s_h,a_h} +\underline{\delta}^{\pi}_{h+1}(s_{h+1})+ \MDSc
+ 4H\sqrt{\frac{SL}{n(h,s_h,a_h) + 1}}\right)+H \ind\{\mathcal{E}^{\text{th}\;\complement}_k\}.	
\end{align*}
\end{proof}

\subsection{Final Bound on Theorem~\ref{thm: Regret main result}}\label{sec:pf_final}
Armed with all the supporting lemmas, we present the remaining proof of Theorem~\ref{thm: Regret main result}.
\begin{proof}
Recall that in the regret decomposition Eq (\ref{eq: regret decom main}), it remains to bound 
\begin{align*}
\sum_{k=1}^{K} \ind\{\mathcal{G}_k\}\left( V_{1}^{*}(s_1^k) - \barV^{\polbarMk}_{1,k}(s_1^k)+ \barV^{\polbarMk}_{1,k}(s_1^k) - V_{1}^{\pi^k}(s_1^k)\right)\nonumber.
\end{align*}

Again, we would skip notation dependence on $k$ when it is clear. For each episode $k$, it suffices to bound \begin{align}
&~\ind\{\mathcal{G}_k\}\left( V_{1}^{*}(s_1) - \barV^{\polbarMk}_{1}(s_1)+ \barV^{\polbarMk}_{1}(s_1) - V_{1}^{\pi}(s_1)\right) \nonumber \\
\leq&~  \ind\{\mathcal{G}_k\}\left[V_{1}^{*}(s_1) - \barV_{1}(s_1)\right] + \ind\{\mathcal{G}_k\}\left[\barV_{1}(s_1) - V^{\pi}_{1}(s_1)\right]\nonumber\\
=&~ \ind\{\mathcal{G}_k\}\deltaEO{1} + \ind\{\mathcal{G}_k\}\deltaEPi{1}.\label{eq: thm proof 1}
\end{align}

We first use Lemma~\ref{lem: pessimism decomp} to relax the first term in Eq~(\ref{eq: thm proof 1}). Applying Eq~(\ref{eq:lem9-1-a}) in Lemma~\ref{lem: pessimism decomp} gives us the following
\begin{align}
&~\ind\{\mathcal{G}_k\}\deltaEO{1}\nonumber\\
=&~ \ind\{\mathcal{G}_k\}\left[V_{1}^{*}(s_1) - \barV_{1}(s_1)\right]\nonumber\\ 
\leq&~ C \ind\{\mathcal{G}_k\}\left(V_{1}^{\pi}(s_1)-\underV_{1}(s_1)+\barV_{1}(s_1) - V^{\pi}_{1}(s_1) + \MDSfind{1}\right)\nonumber\\
=&~ C \ind\{\mathcal{G}_k\}\left(\deltaEPi{1} + \deltaPiU{1}+\MDSfind{1}\right).
\label{eq: bounds on pes 3}
\end{align}

Combining Eq~(\ref{eq: bounds on pes 3}) and Eq~(\ref{eq: thm proof 1}), we get 
\begin{align}
&~\ind\{\mathcal{G}_k\}\left( V_{1}^{*}(s_1) - \barV^{\polbarMk}_{1}(s_1)+ \barV^{\polbarMk}_{1}(s_1) - V_{1}^{\pi}(s_1)\right) \nonumber \\
\leq&~  (C+1)\ind\{\mathcal{G}_k\}\deltaEPi{1} + C\ind\{\mathcal{G}_k\}\left(\mathcal{M}_1^w+\deltaPiU{1}\right).\label{eq:decdd}
\end{align}

We will bound first and second term in Eq~(\ref{eq:decdd}) correspondingly. In the sequence, we always consider the case that Lemma~\ref{lem:est-underV} and Lemma~\ref{lem: estimation decomp} hold. Therefore, the following holds with probability at least $1-\delta/4-\delta/4=1-\delta/2$.

For the $\deltaPiU{1}$ term in Eq~(\ref{eq:decdd}), applying Eq~(\ref{eq:lem9-2})  in Lemma~\ref{lem:est-underV} yields
\begin{align}\label{eq:bound_est_1}
&~ \ind\{\mathcal{G}_k\}\deltaPiU{1}\nonumber\\
=&~\ind\{\mathcal{G}_k\}\left[V_1^{\pi}(s_1)-\underV_1(s_1)\right] \nonumber\\
\leq&~\ind\{\mathcal{G}_k\}\ind\{\mathcal{E}_{1,k}^{\text{th}}\}\left(\envert{\PDiffind{1}+\RDiffind{1}} + \envert{\noiseUind{1}} +\underline{\delta}^{\pi}_{2}(s_{2})+ \MDScind{1} 
+ 4H\sqrt{\frac{SL}{n(1,s_1,a_1) + 1}}\right)+H \ind\{\mathcal{E}^{\text{th}\;\complement}_k\}.
\end{align}

For the $\deltaEPi{1}$ term Eq~(\ref{eq:decdd}), applying Lemma~\ref{lem: estimation decomp} yields
\begin{align}
&~ \ind\{\mathcal{G}_k\}\deltaEPi{1}\nonumber \\
=&~ \ind\{\mathcal{G}_k\}\left[\barV_{1}(s_1) - V^{\pi}_{1}(s_1)\right] \nonumber\\
\leq& ~\ind\{\mathcal{G}_k\}\ind\{\mathcal{E}_{1,k}^{\text{th}}\}\left(\left|\PDiffind{1}+\RDiffind{1} \right| +\noiseind{1} +\overline{\delta}^{\pi}_{2}(s_{2})+ \MDSbind{1}
+  4H\sqrt{\frac{SL}{n(1,s_1,a_1) + 1}}\right)+H \ind\{\mathcal{E}^{\text{th}\;\complement}_k\}.
\label{eq:bound_est_2}
\end{align}

Plugging Eq~(\ref{eq:bound_est_1}) and (\ref{eq:bound_est_2}) into Eq~(\ref{eq:decdd}) gives us, with probability at least $1-\delta/2$,
\begin{align}
&~\ind\{\mathcal{G}_k\}\left( V_{1}^{*}(s_1) - \barV^{\polbarMk}_{1}(s_1)+ \barV^{\polbarMk}_{1}(s_1) - V_{1}^{\pi}(s_1)\right) \nonumber \\
\leq&~  (C+1)\ind\{\mathcal{G}_k\}\deltaEPi{1} + C\ind\{\mathcal{G}_k\}\left(\mathcal{M}_1^w+\deltaPiU{1}\right)\nonumber\\
\leq&~ C\ind\{\mathcal{G}_k\}\ind\{\mathcal{E}_{1,k}^{\text{th}}\}\left(\envert{\PDiffind{1}+\RDiffind{1}} + \envert{\noiseUind{1}} +\underline{\delta}^{\pi}_{2}(s_{2})+ \MDScind{1} 
+  4H\sqrt{\frac{SL}{n(1,s_1,a_1) + 1}}\right)+CH \ind\{\mathcal{E}^{\text{th}\;\complement}_k\}\nonumber\\
& + (C+1)\ind\{\mathcal{G}_k\}\ind\{\mathcal{E}_{1,k}^{\text{th}}\}\left(\left|\PDiffind{1}+\RDiffind{1}\right| + \noiseind{1}  +\overline{\delta}^{\pi}_{2}(s_{2})+ \MDSbind{1}
+  4H\sqrt{\frac{SL}{n(1,s_1,a_1) + 1}}\right) \nonumber\\
&+(C+1)H \ind\{\mathcal{E}^{\text{th}\;\complement}_k\}+ C\ind\{\mathcal{G}_k\}\MDSfind{1}\nonumber\\
=&~C\ind\{\mathcal{G}_k\}\deltaPiU{2}+(C+1)\ind\{\mathcal{G}_k\}\deltaEPi{2}+ C\ind\{\mathcal{G}_k\}\MDSfind{1}\nonumber\\
& + C\ind\{\mathcal{G}_k\}\ind\{\mathcal{E}_{1,k}^{\text{th}}\}\left(\envert{\PDiffind{1}+\RDiffind{1}} + \envert{\noiseUind{1}} + \MDScind{1} 
+  4H\sqrt{\frac{SL}{n(1,s_1,a_1) + 1}}\right)+CH \ind\{\mathcal{E}^{\text{th}\;\complement}_k\}\nonumber\\
& + (C+1)\ind\{\mathcal{G}_k\}\ind\{\mathcal{E}_{1,k}^{\text{th}}\}\left(\left|\PDiffind{1}+\RDiffind{1}\right| + \noiseind{1}  + \MDSbind{1}
+  4H\sqrt{\frac{SL}{n(1,s_1,a_1) + 1}}\right)+(C+1)H \ind\{\mathcal{E}^{\text{th}\;\complement}_k\}.
\label{eq:one_step_dec}
\end{align}

Keep unrolling Eq~(\ref{eq:one_step_dec}) to timestep $H$ and noticing $\deltaPiU{H+1}=\deltaEPi{H+1}=0$ and $\MDSfind{H+1}=0$ yields that with probability at least $1-\delta/2$,
\begin{align}
&~\ind\{\mathcal{G}_k\}\left[V_{1}^{*}(s_1) - V^{\pi}_{1}(s_1)\right] \nonumber \\
\le &~(2C+1)H^2 \ind\{\mathcal{E}^{\text{th}\;\complement}_k\}+C\MDSfind{1}\nonumber\\
&~+C\sum_{h=1}^{H}\ind\{\mathcal{G}_k\}\prod_{h'=1}^h\ind\{\mathcal{E}_{h',k}^{\text{th}}\}\left(\envert{\PDiffh+\RDiffh}+ \envert{\noisehU} + \MDSc  +4H\sqrt{\frac{SL}{n(h,s_h,a_h)+1}}\right)\nonumber\\
&~+\del{C+1}\sum_{h=1}^{H}\ind\{\mathcal{G}_k\}\prod_{h'=1}^h\ind\{\mathcal{E}_{h',k}^{\text{th}}\}\left(\envert{\PDiffh+\RDiffh}  + \noiseh + \MDSb+4H\sqrt{\frac{SL}{n(h,s_h,a_h) + 1}}\right).
\label{eq:final_+dec}
\end{align}
It suffices to bound each individual term in Eq~(\ref{eq:final_+dec}) and we will take sum over $k$ outside.

Lemma~\ref{lem: estimation error} gives us the bound on transition function and reward function
\begin{equation*}
    \sum_{k=1}^K\sum_{h=1}^{H}\ind\{\mathcal{G}_k\} \envert{\PDiffhk+\RDiffhk} = \tilde{\mathrm{O}}(\sqrt{H^3SAT}). 
\end{equation*}

Following the steps in Lemma~\ref{lem: estimation error}, we also get the bound
\begin{equation}
    \sum_{k=1}^K\sum_{h=1}^H H\sqrt{\frac{SL}{n^k(h,s^k_h,a^k_h)+1}} = \tilde{\mathrm{O}}\del{H^{\nicefrac{3}{2}}S\sqrt{AT}}.
\end{equation}

Lemma~\ref{lem: MDS concentration} bounds the martingale difference sequences. Replacing $\delta$ by $\delta'$ in Lemma~\ref{lem: MDS concentration} gives us that with probability at least $1-\delta'$,
\begin{align*}
     &\envert{\sum_{k=1}^K\ind\{\mathcal{G}_k\}\sum_{h=1}^{H}\prod_{h'=1}^h\ind\{\mathcal{E}_{h',k}^{\text{th}}\}\MDSck} = \tilde{\mathrm{O}}(H\sqrt{T})\\
     &\envert{\sum_{k=1}^K\ind\{\mathcal{G}_k\}\sum_{h=1}^{H}\prod_{h'=1}^h\ind\{\mathcal{E}_{h',k}^{\text{th}}\}\MDSbk} = \tilde{\mathrm{O}}(H\sqrt{T})\\
     &\envert{\sum_{k=1}^K\ind\{\mathcal{G}_k\}\MDSfkind{1}} = \tilde{\mathrm{O}}(H\sqrt{T}).
\end{align*}

For the noise term, we first notice that under event $\mathcal{G}_k$, $w^k_{h,s^k_h,\pi^k(s^k_h)}$ can be upper bounded by $\envert{\underline{w}^k_{h,s^k_h,\pi^k(s^k_h)}}$. Applying Lemma~\ref{lem: estimation non random noise} and (replacing $\delta$ by $\delta'$ in Lemma~\ref{lem: estimation non random noise}) gives us, with probability at least $1-2\delta'$
\begin{align*}
&\sum_{k=1}^K\sum_{h=1}^{H}\ind\{\mathcal{G}_k\} w^k_{h,s^k_h,\pi^k(s^k_h)} = \tilde{\mathrm{O}}(H^2S\sqrt{AT})
\end{align*}
and
\begin{align*}
&\sum_{k=1}^K\sum_{h=1}^{H}\ind\{\mathcal{G}_k\}
\envert{\underline{w}^k_{h,s^k_h,\pi^k(s^k_h)}} = \tilde{\mathrm{O}}(H^2S\sqrt{AT}).        
\end{align*}

The warm-up regret term is bounded in Lemma~\ref{lem: warmup bound} 
\begin{align*}
    &~H^2\sum_{k=1}^{K} \ind\{\mathcal{E}^{\text{th}\;\complement}_k\}=\tilde{\mathrm{O}}(H^5S^2A).
\end{align*}

Putting all these pieces together and setting $\delta'=\delta/12$ yields, with probability at least $1-\delta$, we get
\begin{equation}\label{eq: high probability regret}
    \sum_{k=1}^K \ind\{\mathcal{G}_k\}\left( V_{1}^{*}(s_1) - \barV^{\polbarMk}_{1,k}(s_1)+ \barV^{\polbarMk}_{h,k}(s_1) - V_{1}^{\pi^k}(s_1)\right) = \tilde{\mathrm{O}}\del{H^2S\sqrt{AT}+H^5S^2A}.
\end{equation}
This completes the proof of Theorem~\ref{thm: Regret main result}.
\end{proof}

\section{Bounds on Individual Terms}\label{sec: bounds on individual terms}
\subsection{Bound on the Noise Term}
\begin{lemma}\label{lem: estimation non random noise}
With $\underline{w}^k_{h,s^k_h,a^k_h}$ as defined in Definition~\ref{def: under V} and $a^k_h=\pi^k(s^k_h)$, the following bound holds:
\begin{equation*}
    \sum_{k=1}^K\sum_{h=1}^{H}\ind\{\mathcal{G}_k\} \envert{\underline{w}^k_{h,s^k_h,\pi^k(s^k_h)}} = \tilde{\mathrm{O}}\del{H^2S\sqrt{AT}} .
\end{equation*}
\end{lemma}
\begin{proof}
We have:
\begin{align*}
&\sum_{k=1}^K\sum_{h=1}^{H}    \envert{\underline{w}^k_{h,s_h^k,\pi^k(s_h^k)}} =\sqrt{\frac{\beta_kL}{2}}\sum_{k=1}^K\sum_{h=1}^{H}\sqrt{\frac{1}{n^k(h,s_h^k,a_h^k)+1}}
= \sqrt{\frac{\beta_kL}{2}} \sum_{h,s,a}\sum^{n^K(h,s,a)}_{n=1}\sqrt{\frac{1}{n}}.
\end{align*}
Upper bounding by integration followed by an application of Cauchy-Schwarz inequality gives:
\begin{align*}
&\underset{h,s,a}{\sum}\vc{\sum}{n^K(h,s,a)}{n=1}\sqrt{\frac{1}{n}}
\leq \underset{h,s,a}{\sum}\int^{n^K(h,s,a)}_{0}\sqrt{\frac{1}{ x}} dx  = 2\sum_{h,s,a}\sqrt{n^K(h,s,a)} \leq 2\sqrt{HSA\sum_{h,s,a}n^K(h,s,a)} = \mathrm{O}\del{\sqrt{HSAT}}.
\end{align*}
This leads to the bound of $\mathrm{O}\del{\sqrt{\beta_kL}\sqrt{HSAT}} = \tilde{\mathrm{O}}\del{H^2S\sqrt{AT}}$. 
\end{proof}

\subsection{Bound on Estimation Error}
\begin{lemma}\label{lem: estimation error}
For $a^k_h=\pi^k(s^k_h)$, the following bound holds
\begin{equation*}
    \sum_{k=1}^K\sum_{h=1}^{H}\ind\{\mathcal{G}_k\} \envert{\RHathk-\Rhk + \langle\PHathk- \Phk,V^*_{h+1}\rangle} = \tilde{\mathrm{O}}\del{H^{\nicefrac{3}{2}}\sqrt{SAT}}.
\end{equation*}
\end{lemma}

\begin{proof}
Under the event $\mathcal{G}_k$, the estimated MDP $\hat M^k$ lies in the confidence set defined in Appendix~\ref{sec: notations}. Hence
\begin{equation*}
\envert{\RHathk-\Rhk + \langle\PHathk- \Phk,V^*_{h+1}\rangle} \leq \sqrt{e_{k}(h,s^k_h,a^k_h)},
\end{equation*}
where   $\sqrt{e_{k}(h,s^k_h,a^k_h)} = 
  H\sqrt{ \frac{ \log\left( 2HSA k  \right) }{n^k(h,s^k_h,a^k_h)+1}}$.
  
We bound the denominator as
\begin{align*}
&~\sum_{k=1}^K\sum_{h=1}^{H}\sqrt{\frac{1}{n^k(h,s_h^k,a_h^k)+1}} \\
\leq & ~\sum_{h,s,a}\sum^{n^K(h,s,a)}_{n=1}\sqrt{\frac{1}{n}}\\
\leq&~ \sum_{h,s,a}\,\int^{n^K(h,s,a)}_{0}\sqrt{\frac{1}{ x}}dx \\
\leq &~2\sum_{h,s,a}\sqrt{n^K(h,s,a)} \\
\overset{a}{\leq} &~2\sqrt{HSA\sum_{h,s,a}n^K(h,s,a)} \\
=& ~\mathrm{O}(\sqrt{HSAT}),
\end{align*}
where step $(a)$ follows Cauchy-Schwarz inequality.

Therefore we get
\begin{align*}
    \sum_{k=1}^K\sum_{h=1}^{H}\sqrt{e_{k}(h,s^k_h,a^k_h)} = 
  H\tilde{\mathrm{O}}\del{\sqrt{HSAT}} =  \tilde{\mathrm{O}}\del{H^{\nicefrac{3}{2}}\sqrt{SAT}}.
\end{align*}
\end{proof}

\subsection{Bounds on Martingale Difference Sequences}
\begin{lemma}\label{lem: MDS concentration}
The following martingale difference summations enjoy the specified upper bounds with probability at least $1-\delta$,
\begin{gather*}
     \envert{\sum_{k=1}^K\ind\{\mathcal{G}_k\}\sum_{h=1}^{H}\prod_{h'=1}^h\ind\{\mathcal{E}_{h',k}^{\text{th}}\}\MDSck} = \tilde{\mathrm{O}}(H\sqrt{T})\\
     \envert{\sum_{k=1}^K\ind\{\mathcal{G}_k\}\sum_{h=1}^{H}\prod_{h'=1}^h\ind\{\mathcal{E}_{h',k}^{\text{th}}\}\MDSbk} = \tilde{\mathrm{O}}(H\sqrt{T})\\
     \envert{\sum_{k=1}^K\ind\{\mathcal{G}_k\}\MDSfkind{1}} = \tilde{\mathrm{O}}(H\sqrt{T}).
\end{gather*}
\end{lemma}
Here
$\prod_{h'=1}^h\ind\{\mathcal{E}_{h',k}^{\text{th}}\}\MDSck,\prod_{h'=1}^h\ind\{\mathcal{E}_{h',k}^{\text{th}}\}\MDSbk $ are considered under filtration $\overline{\mathcal{H}}^k_h$, while $\MDSfkind{1}$ is considered under filtration $\mathcal{H}^{k-1}_{H}$. Noticing the definition of martingale difference sequences, we can also drop $\ind\{\mathcal{G}_k\}$ in the lemma statement.
\begin{proof}
This proof has two parts. We show (i) above are summations of martingale difference sequences and (ii) these summations concentrate under the event $\mathcal{G}_k$ due to Azuma-Hoeffding inequality \citep{wainwright_2019}. We only present the proof for $\left\{\prod_{h'=1}^h\ind\{\mathcal{E}_{h',k}^{\text{th}}\}\MDSbk\right\}$ and $\{\MDSfkind{1}\}$, and another one follow like-wise. 

We first consider $\prod_{h'=1}^h\ind\{\mathcal{E}_{h',k}^{\text{th}}\}\MDSbk$ term. Given the filtration set $\overline{\mathcal{H}}^k_h$, we observe that
\begin{align*}
\mathbb{E}\sbr{\ind\{\mathcal{G}_k\}\prod_{h'=1}^h\ind\{\mathcal{E}_{h',k}^{\text{th}}\}\overline{\delta}^{\polbarM^k}_{h+1,k}(s^k_{h+1}) \bigg| \overline{\mathcal{H}}^k_h} = \mathbb{E}\sbr{\ind\{\mathcal{G}_k\}\prod_{h'=1}^h\ind\{\mathcal{E}_{h',k}^{\text{th}}\}\mathbb{E}_{s'\sim\PhPik}\left[\overline{\delta}^{\polbarM^k}_{h+1,k}(s') \right]\bigg| \overline{\mathcal{H}}^k_h}.
\end{align*}
This is because the randomness is due to the random transitions of the algorithms when conditioning on $\overline{\mathcal{H}}^k_h$. Thus we have $\mathbb{E}\sbr{\prod_{h'=1}^h\ind\{\mathcal{E}_{h',k}^{\text{th}}\}\MDSbk \bigg| \overline{\mathcal{H}}^k_h}=0$ and  $\left\{\prod_{h'=1}^h\ind\{\mathcal{E}_{h',k}^{\text{th}}\}\MDSbk\right\}$ is indeed a martingale difference on the filtration set $\overline{\mathcal{H}}^k_h$. 

Under event $\mathcal{G}_k$, we also have $\overline{\delta}^{\polbarM^k}_{h+1,k}(s^k_{h+1})  = \barV_{h+1,k}(s_{h+1}^k) - V^{\pi^k}_{h+1,k}(s_{h+1}^k) \leq 2H$. Applying Azuma-Hoeffding inequality (e.g. \cite{azar2017minimax}), for any fixed $K'\in[K]$ and $H'\in[H]$, we have with probability at least $1-\delta'$, 
\begin{equation*}
    \envert{\sum_{k=1}^{K'}\sum_{h=1}^{H'}\prod_{h'=1}^h\ind\{\mathcal{E}_{h',k}^{\text{th}}\}\MDSbk} \leq H\sqrt{4T\log\del{\frac{2T}{\delta'}}} = \tilde{\mathrm{O}}\del{H\sqrt{T}}.
\end{equation*}
Union bounding over all $K'$ and $H'$, we know the following holds for any $K'\in[K]$ and $H'\in[H]$ with probability at least $1-\delta'$
\begin{equation*}
  \envert{\sum_{k=1}^{K'}\sum_{h=1}^{H'}\prod_{h'=1}^h\ind\{\mathcal{E}_{h',k}^{\text{th}}\}\MDSbk} \leq H\sqrt{4T\log\del{\frac{2T}{\delta'}}} = \tilde{\mathrm{O}}\del{H\sqrt{T}}.
\end{equation*}\\

Then we consider $\MDSfkind{1}$ term. Given filtration $\mathcal{H}^{k-1}_{H}$, we know that $\tilV_{1,k}$ has identical distribution as $\barV_{1,k}$. Therefore, for any state $s$, we have
\begin{equation*}
\mathbb{E}\sbr{\ind\{\tilde{\mathcal{G}}_k\}\tilV_{1,k}(s)\bigg| \mathcal{H}^{k-1}_{H}}=\mathbb{E}\sbr{\ind\{\mathcal{G}_k\} \barV_{1,k}(s)\bigg| \mathcal{H}^{k-1}_{H}}.
\end{equation*}

Besides, from the definition of $\mathbb{E}_{\tilde w}$ and $\tilde w$ is the only randomness given $\mathcal{H}^{k-1}_{H}$, we have that for any state $s$,
\begin{align*}
    &~\mathbb{E}\sbr{\ind\{\mathcal{G}_k\}\mathbb{E}_{\tilde w|\tilde{\mathcal{G}}_k}\left[ \tilV_{1,k}(s)\right]\bigg| \mathcal{H}^{k-1}_{H}}\\
    &~\mathbb{E}\sbr{\ind\{\mathcal{G}_k\}\mathbb{E}_{\tilde w|\tilde{\mathcal{G}}_k}\left[\ind\{\tilde{\mathcal{G}}_k\} \tilV_{1,k}(s)\right]\bigg| \mathcal{H}^{k-1}_{H}}\\
    =&~ \mathbb{E}\sbr{\ind\{\mathcal{G}_k\}\mathbb{E}_{\tilde w|\tilde{\mathcal{G}}_k}\left[\ind\{\tilde{\mathcal{G}}_k\} \tilV_{1,k}(s)\big| \mathcal{H}^{k-1}_{H}\right]\bigg| \mathcal{H}^{k-1}_{H}}\\
    =&~ \mathbb{E}\sbr{\ind\{\mathcal{G}_k\}\mathbb{E}\left[ \ind\{\tilde{\mathcal{G}}_k\}\tilV_{1,k}(s)\big| \mathcal{H}^{k-1}_{H}\right]\bigg| \mathcal{H}^{k-1}_{H}}\\
    =&~\mathbb{E}\sbr{\ind\{\tilde{\mathcal{G}}_k\} \tilV_{1,k}(s)\bigg| \mathcal{H}^{k-1}_{H}}.
\end{align*}
Combining these two equations and setting $s=s_1^k$, we have $\mathbb{E}\sbr{\MDSfkind{1} \big| \mathcal{H}^{k-1}_{H}}$. Therefore the sequence $\{\MDSfkind{1}\}$ is indeed a martingale difference.

Under event $\mathcal{G}_k$, we also have $\left|\mathbb{E}_{w|\tilde{\mathcal{G}}_k}\left[ \barV_{1,k}(s^k_1)\right]- \barV_{1,k}(s^k_1)\right| \leq 2H$ from Lemma~\ref{lem: est Q function bounded}. Applying from Azum-Hoeffding inequality (e.g. \cite{azar2017minimax}) and similar union bounding argument above, for any $K'\in[K]$, with probability at least $1-\delta'$, we have
\begin{equation*}
   \envert{ \sum_{k=1}^{K'}\MDSfkind{1}} \leq H\sqrt{4T\log\del{\frac{2T}{\delta'}}} = \tilde{\mathrm{O}}\del{H\sqrt{T}}.
\end{equation*}
The remaining results as in the lemma statement is proved like-wise. Finally let $\delta'=\delta/3$ and uniform bounding over these 3 martingale difference sequences completes the proof.
\end{proof}

\subsection{Bound on the Warm-up Term}\label{sec: bounds on warm-up term}
\begin{lemma}[Bound on the warm-up term]\label{lem: warmup bound}
\begin{align*}
    &\sum_{k=1}^{K} \ind\{\EthreskC \}
     = \tilde{\mathrm{O}}(H^3S^2A).
\end{align*}
\end{lemma}
\begin{proof}
\begin{align*}
    &~\sum_{k=1}^{K} \ind\{\EthreskC\}\\
    = &~ \sum_{k=1}^{K} \ind\{ \underset{h\in[H]}{\cup} n^k(h,s,a) \leq  \alpha_k,\,\forall (h,s,a) = (h,s^k_h,a^k_h)   \} \,\\
    \leq&~ \sum_{k=1}^{K} \sum_{h=1}^{H} \ind\{  n^k(h,s,a) \leq  \alpha_k,\,\forall (h,s,a) = (h,s^k_h,a^k_h) \} \,\\
    \overset{a}{\leq}&~ \sum_{a\in \mathcal{A}} \sum_{s\in \mathcal{S}} \sum_{h=1}^{H}  \alpha_k\\ 
    \leq&~  4H^3S^2A\log\left(2HSAK\right)\decompCL \\
    =&~ \tilde{\mathrm{O}}(H^3S^2A).
\end{align*}

Step $(a)$ is by substituting the value of $\alpha_k$ followed by upper bound for all $4H^3S\log(2HSAK)\decompCL $.
\end{proof}

\end{document}